\documentclass{article}

\usepackage[preprint]{neurips_2022}

\usepackage[utf8]{inputenc} 
\usepackage[T1]{fontenc}    
\usepackage{microtype}
\usepackage{bm}
\usepackage{amsfonts}       
\usepackage{booktabs} 
\usepackage{amsthm}
\usepackage{mathtools}
\usepackage{times}
\usepackage{soul}
\usepackage{algorithm}
\usepackage{algorithmic}
\usepackage{enumerate}
\usepackage{comment}
\usepackage{enumitem}
\usepackage{xspace}    
\usepackage{balance}
\usepackage{natbib}
\usepackage{amssymb,graphicx,url,amsmath}
\usepackage{amsthm}
\usepackage{hyperref}
\usepackage{fontawesome}
\usepackage{nicefrac}
\usepackage{multirow}
\usepackage{subcaption}
\usepackage{cleveref}
\usepackage[table]{xcolor}

\crefformat{section}{\S#2#1#3} 
\crefformat{subsection}{\S#2#1#3}
\crefformat{subsubsection}{\S#2#1#3}

\newcommand{\note}[1]{{\color{red}#1}}

\graphicspath{./}

\def\bi{\begin{itemize}}
\def\ei{\end{itemize}}

\def\calX{\mathcal{X}}
\def\calbfX{\pmb{\mathcal{X}}}
\def\calY{\mathcal{Y}}

\def\bR{\mathbb{R}}

\def\bE{\mathbb{E}}

\def\bfX{\pmb{X}}	
\def\bfx{\pmb{x}}	
\def\bfz{\pmb{z}}	

\def\CoarseAny{\text{\texttt{COARSE-ATTRIB}}\xspace}
\def\FineAny{\text{\texttt{FINE-ATTRIB}}\xspace}

\newtheorem{theorem}{Theorem}
\title{Explaining the root causes of unit-level changes}
\author{%
	Kailash Budhathoki \\
	\And
	George Michailidis\\
	\And
	Dominik Janzing\\
	\AND
	Amazon\\
	\texttt{\{kaibud,michaig,janzind\}@amazon.com}
}

\begin{document}
	\maketitle
	
	\begin{abstract}
		Existing methods of explainable AI and interpretable ML cannot explain change in the values of an output variable for a statistical unit in terms of the change in the input values and the change in the ``mechanism'' (the function transforming input to output). We propose two methods based on counterfactuals for explaining unit-level changes at various input granularities using the concept of Shapley values from game theory. 
		These methods satisfy two key axioms desirable for any unit-level change attribution method. Through simulations, we study the reliability and the scalability of the proposed methods. We get sensible results from a case study on identifying the drivers of the change in the earnings for individuals in the US.
	\end{abstract}

	\section{Introduction}\label{section:introduction}
Changes in the values of an output variable for a statistical unit (e.g., individual) are common when the unit's characteristics change or the system in which those variables are measured changes. 
In finance, a credit scoring model may predict a high credit worthiness score for an individual who had a very low score just a year ago. It could be that the system was upgraded with a new model or the individual simply accumulated a lot of wealth within a year. To be able to take corrective actions---if needed---when we observe a change in the observed target value for a unit, we need to understand what \emph{caused} the change in the first place.

In recent years, several techniques have been developed in interpretable ML (or explainable AI) literature to explain the predicted value in the foreground for a unit relative to some background value, e.g., prediction for another unit~\citep{lundberg:2017:shap,singal:2021:xai-rsv,sundararajan:2017:xai-axioms, frye:2020:xai-asymmetric-shapley,ribeiro:2016:xai-lime, ribeiro:2018:xai-anchors,shrikumar:2017:xai-DeepLIFT,heskes:2020:xai-causal-shap}. These methods can explain the change in the output values in terms of the change in the values of each input variable. But they cannot explain the change in terms of the change in the models, simply because they assume that the model is fixed. The models in ML are just a special case of ``mechanisms''.
 
Formally, our problem setup is as follows. For a unit, suppose a deterministic mechanism $f$ transforms the input vector $\bfx$ to an output value $y$. We observe the output value $y^{(1)}$ in the background, and $y^{(2)}$ in the foreground. We assume that deterministic mechanisms $f^{(k)}$ map respective input $\bfx^{(k)}$ to output values $y^{(k)}$, with $k=1, 2$.  Our goal is to quantify how much of the output change $y^{(2)} - y^{(1)}$ is driven by \textbf{1/} the change in the mechanism ($f^{(1)} \to f^{(2)}$), and \textbf{2/} the change in the input vector ($\bfx^{(1)} \to \bfx^{(2)}$) or further break it down to the change in the value of each input variable ($x_j^{(1)} \to x_j^{(2)}$). 

Our main contributions are as follows.
\textbf{1/} We first axiomatize desirable properties of an attribution method for explaining unit-level changes (\cref{section:axioms}).
\textbf{2/} We then propose two methods based on counterfactuals for attributing unit-level changes at various input granularities using the concept of Shapley values from game theory (\cref{section:explain}). In particular, the first method attributes a unit-level change to the change in the input vector and the change in the mechanism from the background to the foreground scenario (\cref{sec:coarse-any}). This coarse-grained attribution can be broken down to a finer detail up to the change in the value of each input variable for linear mechanisms (\cref{sec:fine-linear}). To obtain a fine-grained attribution for non-linear mechanisms, we modify the coarse-grained attribution method to include the change in the value of each input variable (\cref{sec:fine-any}). We discuss their properties in \cref{section:properties}.

We compare related work in \cref{section:related-work}. Through simulations, we study the reliability and scalability of the methods in \cref{subsec:simulations}. Then we present a real-world case study identifying the drivers of the change in the average earnings of US individuals reported through surveys in 1976 and 1982, and discuss the results for an exemplar individual. Finally, we conclude with a discussion on the uncertainty quantification and the limitations of our approach in \cref{section:discussion}. All proofs are in the appendix. We also present an attribution method when one demands explanations of unit-level changes w.r.t. causal mechanisms of variables in the real-world that are stochastic in the Appendix. 

%
	\section{Problem Definition}\label{section:problem}
%

Let $\bfX \coloneqq (X_1, \dotsc, X_d)$ denote the input variables (e.g., income, marital status) to a deterministic mechanism $f$ that produces the output variable $Y$ (e.g., credit score). 
We will use $\calbfX$ to denote the space of input values, and $\calY$ for the space of output values. Typically, $\calbfX=\bR^m$ and $\calY=\bR$. 
In the context of ML, the mechanism $f$ can be a prediction model.
Formally, we make the following assumption:

\textbf{Assumption A1/} A deterministic algorithm $f:\calbfX \to \calY$ produces output $Y$ from inputs $\bfX$. 


The output $Y$ produced by $f$ can differ from the actual variable $\tilde{Y}$ in the real-world. In particular, the inputs $\bfX$ \emph{cause} the output $Y$ w.r.t. the algorithm $f$. In a causal language, inputs $\bfX$ are the root nodes with outgoing arrows pointing to $Y$. The causal relationships between variables $X_1, \dotsc, X_d, Y$ in the real-world become irrelevant here as the output $Y$ results from applying the algorithm $f$ on inputs $X$~\citep{janzing:2020:feature}. We also refer to $f$ as the mechanism that generates $Y$ from $\bfX$.

For a statistical unit (e.g., a person), suppose that we observe the output value $y^{(1)}$ in the background (e.g., \emph{score} in 1976) and $y^{(2)}$ in the foreground (e.g., \emph{score} in 1982). The background and foreground scenarios can be two snapshots of time. 
We want to explain the drivers of the change in the output value of the unit, i.e., $\Delta y \coloneqq y^{(2)} - y^{(1)}$.
Let us rewrite $\Delta y$ in terms of input values and mechanisms:
\begin{align}
	\Delta y \coloneqq f^{(2)}\left(\bfx^{(2)}\right) - f^{(1)}\left(\bfx^{(1)}\right) .
\end{align}
Observe that the changes in the input ($\bfx^{(1)} \to \bfx^{(2)}$) and the mechanism ($f^{(1)} \to f^{(2)}$) drive $\Delta y$. Our goal is thus to attribute $\Delta y$ to \textbf{1/} the mechanism $f$, and \textbf{2/} the inputs $\bfX$ overall or break it down to each input variable $X_j$.
By working with deterministic mechanisms $f$, we keep the framework general for a broad range of AI/ML applications. In the context of explainable AI or interpretable ML, we can explain the drivers of the change in the predicted values between two scenarios. 

\section{Desirable Properties}\label{section:axioms}
Next we discuss desirable properties of an attribution method for explaining unit-level changes.

\textbf{Axiom 1/ Dummy:} An attribution method for explaining unit-level changes satisfies Dummy property if the causal driver gets a zero attribution if its value does not change from the foreground scenario to the background scenario. In particular, the Dummy property implies the following:
\textbf{A/} the mechanism $f$ gets a zero attribution if the foreground and background mechanisms are equivalent, i.e., $f^{(2)}(\bfx) = f^{(1)}(\bfx)$ for all $\bfx \in \calbfX$, and
\textbf{B/} the overall input $\bfX$ gets a zero attribution if the foreground and background input vectors are the same, i.e., $\bfx^{(2)} = \bfx^{(1)}$, or any input variable $X_j$ gets a zero attribution if the foreground and background values of $X_j$ are the same, i.e., $x_{j}^{(2)} = x_{j}^{(1)}$.

The Dummy property reflects the semantics that a cause does not contribute to the change in the effect, if the cause did not change. For example, it is plausible that the education level of individuals is one of the causes of their wage; if the education level of an individual does not change between two years but the wage changes, then education level does not contribute to the wage-change. This axiom formalizes the Dummy axiom of Shapley values~\citep{shapley:1953:solution} in the context of unit-level changes.

\textbf{Axiom 2/ Completeness:} An attribution method for explaining unit-level changes satisfies Completeness property if the attributions add up to the unit-level output change $\Delta y$. Such a linear decomposition is desirable because then the attributions are interpretable to humans. This axiom formalizes the Efficiency axiom of Shapley values~\citep{shapley:1953:solution} for unit-level changes.


\section{Attributing unit-level changes}\label{section:explain}
First we present our method to attribute the output change $\Delta y$ for a unit $u$ to the mechanism $f$ and the overall input $\bfX$ in \cref{sec:coarse-any}. For the special case of linear mechanisms, we derive a closed-form solution in \cref{sec:fine-linear}. Through the decomposition of that closed-form solution, we also obtain a fine-grained attribution to each input variable $X_j$. Finally, to obtain a fine-grained attribution for non-linear mechanisms, we introduce our second method in \cref{sec:fine-any}. Whenever we mention the causes of $\Delta y$, we interchangeably use the scalar placeholder $x_j$ for $X_j$ and the vector placeholder $\bfx$ for $\bfX$.


\subsection{Coarse-grained attribution for \emph{any} type of mechanism}\label{sec:coarse-any}
Here we attribute the output change $\Delta y$ to the mechanism $f$ and the inputs $\bfX$. We do not break down the attribution to $\bfX$ to each input variable $X_j$. In particular, the set of causes of $\Delta y$ is $\{f, \bfx\} \eqqcolon W$. Here, we are agnostic to the type of mechanism $f$.

Observe that we can obtain the foreground output value $y^{(2)} \coloneqq f^{(2)}(\bfx^{(2)})$ from the background value $y^{(1)} \coloneqq f^{(1)} (\bfx^{(1)})$ by gradually replacing the input vector $\bfx$ and mechanism $f$ to their foreground values one by one. Each time we replace a cause by its foreground value, we can compute its contribution as the change in the output value relative to the output value obtained by replacements before. In particular, we can compute the contribution of a cause $w \in W$ given that we have already replaced causes in subset $A \subset W$ to their foreground values before as 
\begin{align}
	C(w \mid A; W) &\coloneqq f^{\left(\mathbf{2}_{A \cup \{w\}}(f)\right)} \left(\bfx^{\left(\mathbf{2}_{A \cup \{w\}}(\bfx)\right)}\right) - 
	 f^{\left(\mathbf{2}_{A}(f)\right)} \left(\bfx^{\left(\mathbf{2}_{A}(\bfx)\right)}\right), \label{eq:contrib-coarse}
\end{align}
where $\mathbf{2}_A$ is an indicator function of subset $A$ that tells whether a cause $v \in W$ assumes its foreground value, indexed by (2), or its background value, indexed by (1), based on the causes in $A$, i.e.,
\begin{align}
	\mathbf{2}_A(v) \coloneqq 
	\begin{cases}
		2 \text{ if } v \in A,\\
		1 \text{ if } v \notin A.
	\end{cases}
\end{align}
But the contribution depends on the subset $A$ given as context.
For any ordering $W_{\sigma(1)}, \dotsc, W_{\sigma(|W|)}$ of causes $W$, we could consider the contribution of $W_{\sigma(j)}$ given $W^{\mathit{ctx}}_{\sigma(j)} = \{W_{\sigma(1)}, \dotsc, W_{\sigma(j-1)}\}$ as the context. This dependence on ordering introduces arbitrariness in the attribution procedure.
%

To get rid of the arbitrariness, we leverage Shapley values~\citep{shapley:1953:solution} from cooperative game theory. The key idea of Shapley values is to symmetrize over all orderings, i.e.,
consider all possible orderings, compute the contribution for each ordering, and then take the average.

Using the concept of Shapley values, the contribution of a cause $w \in W$ to the output change $\Delta y$ is then given by the average contribution over all possible orderings, i.e.,
\begin{align}
	\pi(w) \coloneqq \frac{1}{|W|!} \sum_{\sigma} C(w \mid W^{\mathit{ctx}}_{\sigma(j)}; W) = \sum_{A \subset W \setminus \{w\}} \frac{1}{|W| \binom{|W|-1}{|A|}} C(w \mid A; W),
\end{align}
where the second summation follows from averaging contributions for permutations of $W$ with the same value of $A$ before $w$ in the ordering.

As we only have two causes (i.e., $|W|=2$), we can derive a closed-form solution with a simple algebraic manipulation. In particular, the Shapley value contribution of the mechanism $f$ equals:
\begin{align} 
	\pi({f}) &\coloneqq \frac{1}{2}\left\{ f^{(2)}\big(\bfx^{(1)}\big) - f^{(1)}\big(\bfx^{(1)}\big) \right\} +
	\frac{1}{2}\left\{ f^{(2)}\big(\bfx^{(2)}\big) - f^{(1)}\big(\bfx^{(2)}\big) \right\} \label{eq:sv-f}
\end{align}

In words, the contribution of $f$ to $\Delta y$ is the average of the output changes (using $f^{(1)}$ vs $f^{(2)}$) for the background and foreground input vectors separately.

Similarly, the Shapley value contribution of the overall input $\bfX$ equals:
\begin{align}
	\pi(\bfx) &\coloneqq \frac{1}{2}\left\{ f^{(1)}\big(\bfx^{(2)}\big) - f^{(1)}\big(\bfx^{(1)}\big) \right\} + 
	\frac{1}{2}\left\{ f^{(2)}\big(\bfx^{(2)}\big) - f^{(2)}\big(\bfx^{(1)}\big) \right\} \label{eq:sv-x}
\end{align}
 In words, the contribution of inputs $\bfX$ to $\Delta y$ is the average of the output changes for the foreground and background input vectors using the background and foreground mechanisms separately.
 
\textbf{Counterfactual interpretation.} In the expression of Shapley values $\pi({f})$ and $\pi(\bfX)$ above, besides two factual terms $f^{(1)}\big(\bfx^{(1)}\big)$ and $f^{(2)}\big(\bfx^{(2)}\big)$, we also have two counterfactual terms $f^{(2)}\big(\bfx^{(1)}\big)$ and $f^{(1)}\big(\bfx^{(2)}\big)$. Thus, the attributions $\pi(f)$ and $\pi(\bfX)$ consider two key counterfactual questions that naturally arise when explaining unit-level changes: \textbf{1/} What would have been the output value for the unit had the input vector remained the same (to its background value), but the mechanism changed (to its foreground value)? The term $f^{(2)}\big(\bfx^{(1)}\big)$ answers this enquiry.
\textbf{2/} What would have been the output value for the unit had the mechanism remained the same (to its background value), but the input vector changed (to its foreground value)? The term $f^{(1)}\big(\bfx^{(2)}\big)$ answers this enquiry. These counterfactuals are predictive, and  fit the predictive interpretation of counterfactuals by Pearl using a structural causal model when the unobserved noise terms remain constant~\citep[Ch.~7.2.2]{pearl:2009:book}.


We refer to this attribution method as \CoarseAny, to emphasize that the attribution $\pi(\bfx)$ is for the overall input variables instead of each input variable, and the procedure works with \emph{any} type of mechanism $f$ (linear or non-linear).
For the special case of linear mechanisms, we can derive a closed-form solution from the above procedure that provides a more fine-grained attribution at the level of each input variable and the mechanism, which we discuss next. 

\subsection{Breakdown of coarse-grained attribution for \emph{linear} mechanisms}\label{sec:fine-linear}

Consider the linear mechanisms of the form:
$f^{(k)}(\bfx^{(k)}) \coloneqq \sum_{j=1}^d \beta_j^{(k)} x_j^{(k)}, \ k=1,2,$
where $\beta_j^{(k)}$'s are the linear coefficients. Then the Shapley value contribution of inputs $\bfX$ to $\Delta y$ (Eq.~\ref{eq:sv-x}) reduces to
\begin{align}
	\pi(\bfx)  \coloneqq \sum_{j=1}^d \frac{(\beta^{(1)}_j + \beta^{(2)}_j)}{2} (x^{(2)}_j - x^{(1)}_j ).
\end{align}
From here, we obtain the contribution of each input variable $X_j$ to $\Delta y$ as 
\begin{align}
    \pi(x_j) \coloneqq \frac{\beta^{(1)}_j + \beta^{(2)}_j}{2} \left(x^{(2)}_j - x^{(1)}_j \right) .
\end{align}
In words, the Shapley value contribution of each input variable $X_j$ captures the impact of the change in the values of the variable in the foreground and background scenarios, weighed by the \textit{average} value of corresponding linear coefficients.

Analogously, the Shapley value contribution of the mechanism $f$ to $\Delta y$ (Eq.~\ref{eq:sv-f}) reduces to
\begin{align}
	\label{eq:SV-player-f-linear}
	\pi(f) &\coloneqq \sum_{j=1}^d \frac{x^{(1)}_j + x^{(2)}_j} {2} \left(\beta^{(2)}_j - \beta^{(1)}_j \right).
\end{align}
That is, the Shapley value contribution of $f$ captures the impact of changes in the linear coefficients of variables in the background and foreground scenarios, weighed by the \emph{average} values of corresponding variables in the two scenarios. 
	
The Shapley values $\pi(x_j)$ and $\pi(f)$ can be negative for the linear mechanisms because we consider the change in the linear coefficients as well as the change in the input values, both of which can be either positive or negative. Consider a simple case wherein the linear coefficients of the foreground and background mechanisms are positive. A negative value of $\pi(x_j)$ then implies a decrease in the value of $X_j$ from the background to the foreground scenario, whereas a negative value of $\pi(f)$ implies a decrease in the linear coefficients from the background to the foreground scenario. 
%
This two-step attribution procedure for fine-grained attribution only works with linear mechanisms, however. Next we present an attribution method that provides a fine-grained attribution for \emph{any} type of mechanism.

\subsection{Fine-grained attribution for \emph{any} type of mechanism}\label{sec:fine-any}
To obtain a fine-grained attribution for any type of mechanism, we consider each input variable as a potential cause of the output change, besides the mechanism. The core idea remains the same as in the coarse-grained attribution: starting with their background values for all causes, we gradually replace the value of each cause by its foreground value to measure its contribution. 

Here, the set of causes of $\Delta y$ is $\{f, x_1, \dotsc, x_d\} \eqqcolon V$. 
We modify the contribution (Eq.~\ref{eq:contrib-coarse}) in the coarse-grained attribution method to account for each input variable. In particular, we define the contribution of a cause $w$ given that we have already replaced causes in $A \subset V$ to their foreground values before as 
\begin{align*}
	C(w \mid A; V) \coloneqq 
	f^{\left(\mathbf{2}_{A \cup \{w\}}(f)\right)} \left(x_1^{\left(\mathbf{2}_{A \cup \{w\}}(x_1)\right)}, \dotsc \right) -
	f^{\left(\mathbf{2}_{A}(f)\right)} \left(x_1^{\left(\mathbf{2}_{A}(x_1)\right)}, \dotsc
	\right).
\end{align*}
Averaging the contributions over all orderings, we get the Shapley value contribution of each cause $w \in V$ to the output change $\Delta y$ as
\begin{align}
    \varphi(w) \coloneqq \sum_{A \subset V\setminus \{w\}} \frac{1}{(d+1)\binom{d}{|A|}} C(w \mid A; V). \label{eq:shapley-fine}
\end{align}

We refer to this attribution method as \FineAny. While this method is close in spirit to the classical Shapley value attribution proposal for explaining prediction models~\citep{lundberg:2017:shap}, the key difference is that we also include the mechanism as a player in the game. For reasons similar to the previous two attribution methods, the Shapley value contributions $\varphi(w)$'s can also be negative. 
It is easy to see that for a single input variable, this attribution method coincides with the coarse-grained attribution; the expressions for Shapley values will be the same. When there are at least two input variables (i.e., $d \geq 2$), we consider more than two counterfactual terms however---involving various combinations of input variables and mechanisms---in this fine-grained attribution.



\subsection{Properties of proposed methods}\label{section:properties}
We are now ready to discuss the properties of the proposed attribution methods. 


\begin{theorem}
	Both attribution methods satisfy completeness and dummy axioms.
\end{theorem}

\begin{theorem}
	When the mechanisms are linear, i.e., $f^{(k)}(\bfx^{(k)}) \coloneqq \sum_{j=1}^d \beta_j^{(k)} x_j^{(k)}, \ k=1,2$,
	the Shapley value attributions of both methods are the same, i.e., $\varphi(w) = \pi(w)$ for $w \in \{f, x_1, \dotsc, x_d\}$.
\end{theorem}

\textbf{Computational Complexity.} Assume that the mechanism in concern $f$ is an oracle that evaluates an input vector $\bfx$ of dimension $d$ in $\mathcal{O}(g(d))$ time for some function $g$. That is, the worst-case computational complexity of $f(\bullet)$ is $\mathcal{O}(g(d))$. This is a reasonable assumption because commonly assumed mechanisms like prediction models take more steps to evaluate an input as its dimensionality increases. The coarse-grained attribution method evaluates input vectors by mechanisms four times in total, and performs basic algebraic operations on those to obtain the attributions. As such, \CoarseAny runs in $\mathcal{O}(g(d))$. The attribution method for linear mechanisms only perform simple algebraic operation on the values of input variables and linear coefficients, and hence has a worse-case complexity of $\mathcal{O}(1)$. The fine-grained attribution method for non-linear mechanisms, however, has an exponential complexity, i.e., $\mathcal{O}(2^d)$, as the contributions---which take constant time---are computed over all subsets whose size is $2^{d}$. Up to 30 input variables, \FineAny finishes within minutes. For more variables, we use sampling approximations to Shapley value contributions~\citep{kononenko:2014:sampling-shap} where we sample a fixed number of subsets and compute the Shapley values by averaging contributions over those (in Eq.~\ref{eq:shapley-fine}).

	\section{Related Work}\label{section:related-work}
Next, we compare and contrast the proposed methods to existing approaches in the explainable AI and interpretable ML literature. To this end, we group the existing literature into three scenarios.

\textbf{1/ Input fixed, Mechanism fixed.}
The vast body of work on explainable AI focuses on explaining a \emph{fixed} AI/ML model~\cite{cohen:1988:xai-evaluating-ai,newell:1976:xai-ablation-study,lundberg:2017:xai-shap,ribeiro:2016:xai-lime,ribeiro:2018:xai-anchors,lakkaraju:2020:xai-black-box-explain,nguyen:2019:xai-activation-maximisation,singal:2021:xai-rsv,lundberg:2017:xai-DeepSHAP, shrikumar:2017:xai-DeepLIFT, sundararajan:2020:xai-BShap}. That is, 
most explanation methods (e.g., node-based, instance-based) explain a single prediction model $f$ at one of the following levels of details.

\textbf{Unit--Input Level.} Methods like LIME~\cite{ribeiro:2016:xai-lime}, SHAP~\cite{lundberg:2017:xai-shap}, causal Shapley value~\citep{heskes:2020:xai-causal-shap}, and asymmetric Shapley value~\cite{frye:2020:xai-asymmetric-shapley} explain some function $g$ of prediction $f(\bfx)$ by quantifying the importance of the value $x_j$ of each input variable for unit $u$. A typical example of $g$ is the deviation of prediction from its expected value, i.e., $g \coloneqq  f(\bfx) - \bE[f(\bfX)]$. While we can introduce a new binary variable $S$ to a new function $f':\calbfX \times \mathcal{S} \rightarrow \calY$, which switches to either $f^{(1)}$ or $f^{(2)}$, doing so introduces ambiguity in selecting the distribution for computing the baseline expectation $\bE[f'(\bfX)]$. 
    
\textbf{Unit Level.} Methods in this category explain some function $g$ of prediction $f(\bfx)$ by quantifying the importance of a subset of input vector $\bfz \subseteq \bfx$ of the unit $u$ or for other units $\tilde{u}$. Counterfactual visual explanation method~\cite{goyal:2019:xai-counterfactual-visual-explanations} finds regions $\bfz$ in a given image $\bfx$, when changed, would produce a different classification. \citet{koh:2017:xai-influence-function} use influence functions to identify instances $\bfx$ in the training set that are responsible for the prediction of a given test instance $\tilde{\bfx}$. Nguyen et al.~\cite{nguyen:2019:xai-activation-maximisation} use activation maximization to identify input instances $\bfx$ that strongly activate a function (neuron) of interest. Anchors~\cite{ribeiro:2018:xai-anchors} explain the prediction $f(\bfx)$ for an input vector $\bfx$ by identifying rules (i.e., subsets $\bfz \subseteq \bfx$) that are sufficient for the prediction. These methods also cannot attribute to the mechanism $f$ as they also assume that the mechanism is fixed. 
    
\textbf{Aggregate Input Level.} Classical proposals on ablation study~\cite{newell:1976:xai-ablation-study,cohen:1988:xai-evaluating-ai} quantify the importance of an input variable to the prediction model, e.g., through random perturbation of the values of input variables or removing the input variables. As such, we obtain the importance of input variable $X_j$ to the prediction model $f$, in contrast to the prediction $f(\bfx)$ for a unit $u$. Methods that provide explanations at the unit--input level can also provide input--level explanations by aggregating unit-level explanations over all units of an input variable. Like previous methods, these also assume that the mechanism is fixed.

\textbf{2/ Input changes, Mechanism fixed.}
Most edge-based or flow-based explanation methods~\cite{singal:2021:xai-rsv,shrikumar:2017:xai-DeepLIFT} quantify the importance of an edge, given a graphical representation of input variables (e.g., causal graph), by propagating changes in the values of input variables along the edges. As such, these methods can handle changes in input variables. Formally, these methods attribute the deviation in the prediction for a unit $u$ from the prediction for a reference unit $\tilde{u}$ using a single prediction model $f$, i.e., $f(\bfx) - f(\tilde{\bfx})$. We could set $\bfx \coloneqq \bfx^{(2)}$ and $\tilde{\bfx} \coloneqq \bfx^{(1)}$ to attribute the prediction change to input variables. But they cannot attribute to the mechanism as they assume that the mechanism is fixed. 

\textbf{3/ Input changes, Mechanism changes.}
A classical statistical methodology called Kitagawa-Blinder-Oaxaca decomposition~\cite{kitagawa:1955:xai-kitagawa-blinder-oaxaca-decomposition} also explains changes, but at the aggregate level (mean), and only works with linear models. In a nutshell, the decomposition of prediction change into input change and mechanism change is obtained by taking the difference between the linear regression equations of background and foreground scenarios.
A recent proposal by \citet{budhathoki:2021:distribution-change} attributes the change in the distribution (or its property) of a variable to the conditional distribution of each variable given its direct parents in the causal graph. If an FCM is also available, by replacing the structural assignments at the unit-level (with corresponding noise values), we can attribute the unit-level output change. We operationalize this idea in Appendix B. 
	\section{Experiments}\label{section:experiments}
First, through simulations, where we can establish the ground truth, we evaluate the performance of our methods for attributing output changes when we learn prediction models from data, and study the scalability of \FineAny whose computational complexity is exponential (\cref{subsec:simulations}). Then we assess whether results are sensible on a real-world case study (\cref{subsec:case-study}). We ran the scalability simulations on a Macbook Pro with 8GB RAM and 1.4 GHz Quad-Core Intel Core i5 processor.

\subsection{Simulated experiments}\label{subsec:simulations}
Here we empirically investigate two research questions: \textbf{1/} How reliable are the attributions when we do not know the underlying mechanisms? \textbf{2/} Does the fine-grained attribution scale?


\textbf{Reliability Simulation Setup \& Results.} We start with ground truth background and foreground models with random parameters. Then we generate random background and foreground input/output samples of the same size. For each instance of those samples, we obtain ground truth attributions using the ground truth models. Likewise, using prediction models learned from respective samples, we obtain estimated attributions with fitted models.

In particular, we consider 2 types of ground truth models: linear and polynomial with interactions. We uniformly randomly choose between 1 to 5 number of inputs. For a polynomial model, we uniformly randomly select its degree between 2 to 4. For both models, we uniformly randomly assign between -5 to 5 to each of its linear coefficients. For each input variable (in both background and foreground scenario), we draw samples of size 2000 independently from a standard Normal distribution. Then we apply those models to obtain the background and foreground outputs.

Fig.~\ref{fig:simulation-results} (a--b) shows the mean absolute error (MAE) of estimated fine-grained and coarse-grained attributions (relative to ground truth attributions) using fitted linear and XGBoost regressions (with default hyperparameters) for 100 ground truth models. When the ground truth is linear, regardless of the type of fitted prediction models, estimated attributions are closer to the ground truth attribution. In contrast, for the non-linear ground truth, attributions show a relatively high variance.

\begin{figure}[tb]
	\centering
	\begin{minipage}{0.25\columnwidth}
		\centering
		\includegraphics[width=0.85\columnwidth]{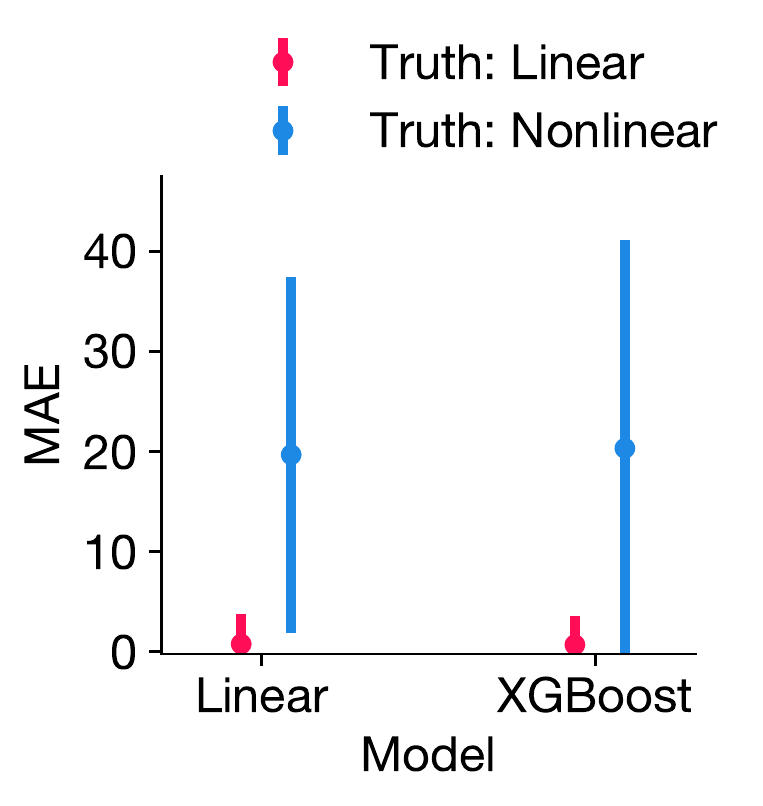}
		\subcaption{\CoarseAny}
	\end{minipage}%
	\begin{minipage}{0.25\columnwidth}
		\centering
		\includegraphics[width=0.85\columnwidth]{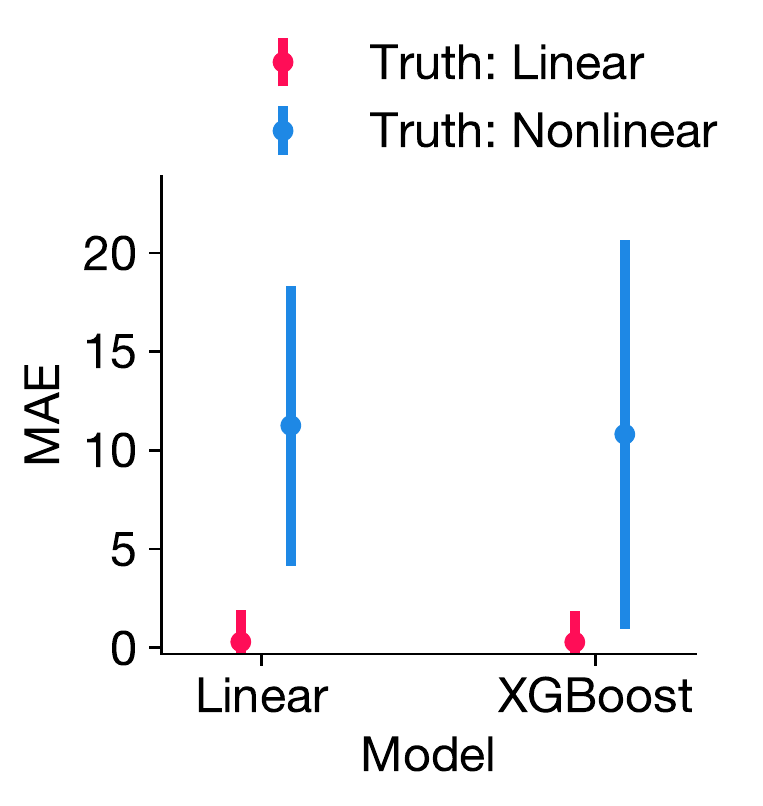}
		\subcaption{\FineAny}
	\end{minipage}%
	\begin{minipage}{0.5\columnwidth}
		\centering
		\includegraphics[width=\columnwidth]{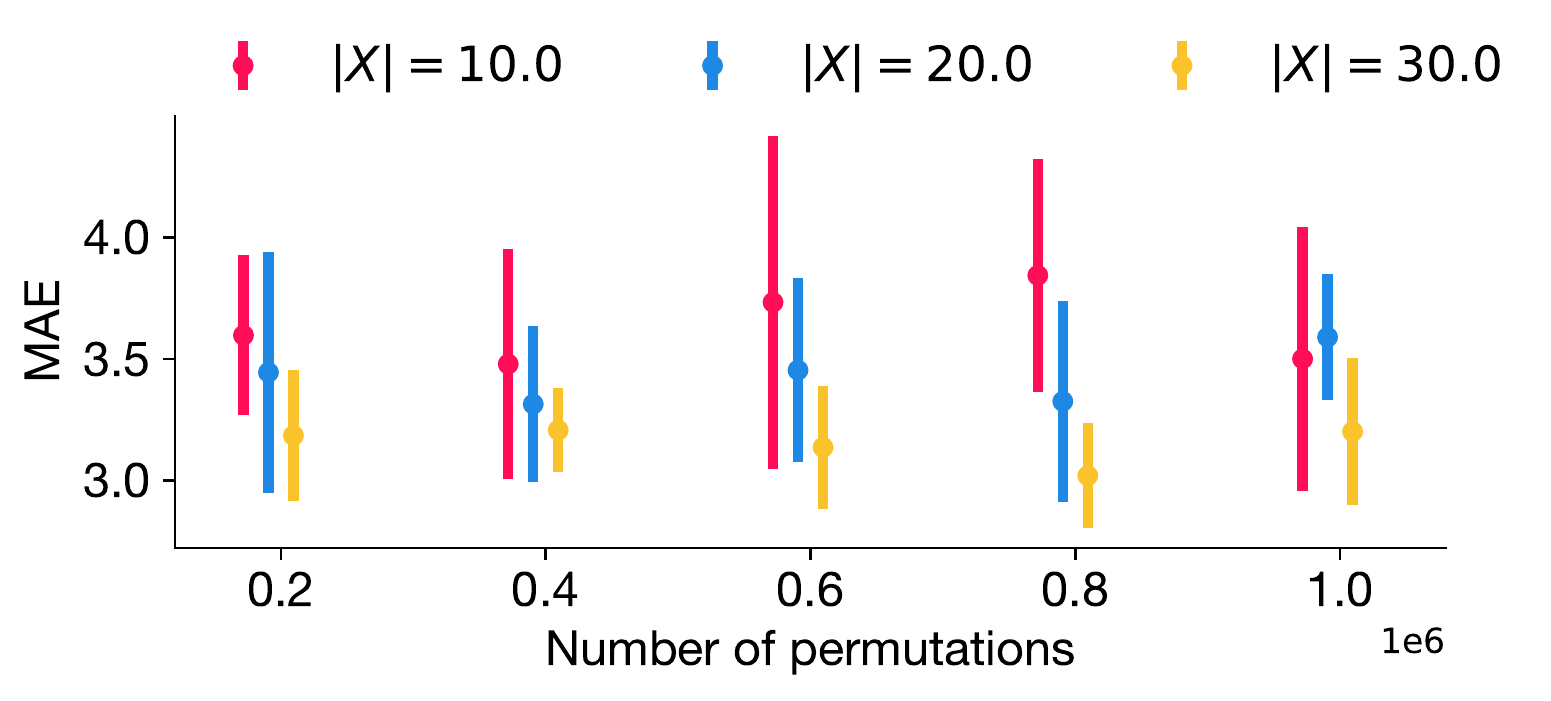}
		\subcaption{\FineAny}
	\end{minipage}
	\caption{\textbf{(a--b)} Mean absolute error (MAE) of estimated fine-grained and coarse-grained attributions from 2 types of prediction models for linear and non-linear ground truth models. Both types of estimated attributions show a large variance for the non-linear ground truth. For the linear ground truth, however, the estimated attributions have lower variance and are also closer to the ground truth attributions. \textbf{(c)} Approximation error of fine-grained attributions for the linear ground truth. As we sample more permutations, the approximation error decreases.}
	\label{fig:simulation-results}
\end{figure}

\textbf{Scalability Simulation Setup \& Results.} We first study how well sampling-based fine-grained attributions approximates exact fine-grained attributions. We use linear ground truth models where we can compute the exact attributions using the closed form expression for any number of features. Fig.~\ref{fig:simulation-results} (c) shows the MAE of approximated attributions for the case of 10, 20 and 30 input variables as we sample more permutations. The reported averages and standard errors are calculated over 100 ground truth models and 1000 sample sizes. As expected, if there are more input variables, with more permutations the attributions are close to the true attributions. This tells us that we can trade-off accuracy for speed to obtain fine-grained attributions when the number of input variables is large.


\subsection{Real-world case study}\label{subsec:case-study}
We will use real-world earnings data to explain how our method can be used to explain the drivers of changes in earnings over time. We first explain the drivers for all individuals in the study (by aggregating over unit--level attributions), and then show the results for an exemplar individual.

\textbf{Background.}
The Panel Study of Income Dynamics (PSID) is an ongoing survey of households launched in 1968 in the United States with the goal of providing data to assess President Lyndon Johnson's ``War on Poverty''~\citep{johnson:2018:psid}. The survey was conducted annually until 1997 and biennially thereafter. In particular, this survey collects information on income and family demographics of thousands of representative households in the US over time. We use our method to identify the drivers of change in the income of households between two surveys, 1976 and 1982.

\textbf{Data.}
We consider a sample of 595 individuals who are heads of the households from the surveys carried out between 1976 and 1982. So, for each individual, we have 7 annual observations on 12 variables related to income. Our target is ``wage'' (reported in natural logarithm). We use the 9 relevant variables as input to a prediction model that predicts ``wage'' (clarified in Appendix A).

\textbf{Modelling.} We wish to understand the drivers of change in the earnings reported in surveys from 1976 versus 1982. We observe that the earnings of roughly 98\% of individuals have increased between the surveys (585 out of 595). 

Let $\Delta \mathrm{wage}\%$ denote the growth in average earnings reported in 1982 relative to that in 1976, i.e.,
\begin{align*}
    \Delta \mathrm{wage} \% = \frac{\Big(\sum_{u=1}^{595}\mathrm{wage}^{(\mathrm{1982})}(u) - \mathrm{wage}^{(\mathrm{1976})}(u)\Big)}{\sum_{u=1}^{595}\mathrm{wage}^{(\mathrm{1976})}(u)} \times 100 \%,
\end{align*}
where $\mathrm{wage}^{(\mathrm{t})}(u)$ is the earnings of individual $u$ reported in the survey conducted in year $t$, and divisions by $595$ cancel out.
We observe a relative growth of $\Delta \mathrm{wage}\% \approx 9\%$. What drove this change? It is difficult to establish the causal graph here. We thus resort to prediction models, and explain this wage gap w.r.t. the prediction model.

We assume a linear mechanism, which we shall probe with model fit diagnostics later. That is, we model ``wage'' as a linear function of input variables for each survey. 
We use Ordinary Least Squares (OLS) method for fitting the linear models from data. The fitted linear coefficients and model diagnostics are shown in Figure~\ref{fig:1976-vs-1982-fit} in the Appendix. We observe that the R-squared fits of linear models from both data are roughly 0.99. That is, roughly 99\% of the variance in ``wage'' is explained by each linear model. This shows that linear models are suitable for this case study. 

To explain $\Delta \mathrm{wage}\%$, first we compute the Shapley value attributions to input variables and mechanism for the change in earnings of each individual $u$, i.e. $\mathrm{wage}^{(\mathrm{1982})}(u) - \mathrm{wage}^{(\mathrm{1976})}(u)$. We then take the sum  of Shapley value attributions to each input variable of all individuals, and repeat the same for the mechanism. The grand sum of the resulting sums then matches the numerator $\sum_{u=1}^{595}\mathrm{wage}^{(\mathrm{1982})}(u) - \mathrm{wage}^{(\mathrm{1976})}(u)$, by which we obtain the aggregate attributions to input variables and mechanism.

\paragraph{\textbf{Results: Aggregate attribution}.}
Our method reveals that roughly 76\% of $\Delta \mathrm{wage} \%$ is driven by mechanism change and the remaining 24\% is driven by input change as shown in Figure~\ref{fig:attribution-results} (a). If we break down the attribution to input change further, the change in the years of full-time work experience (i.e., ``exp'') is the biggest driver amongst others.
\begin{figure}[tb]
    \centering
    \begin{minipage}{0.48\columnwidth}
        \includegraphics[width=\columnwidth]{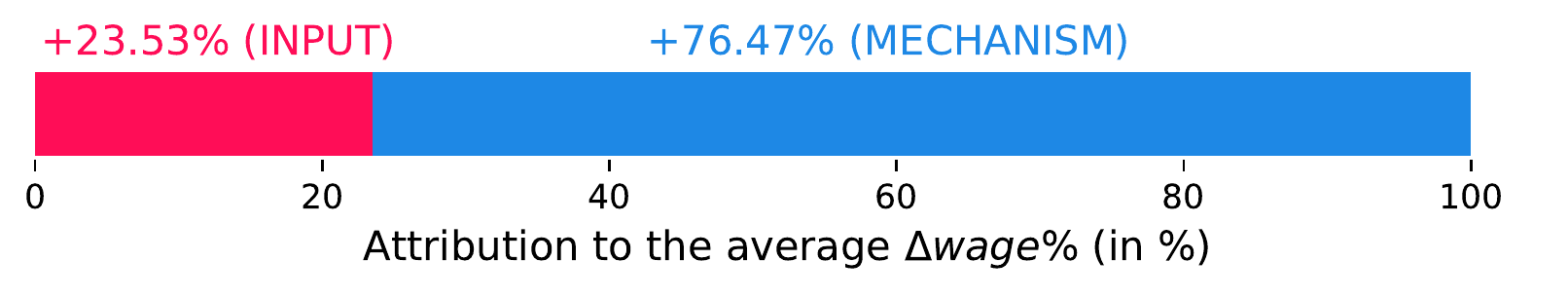}
    \end{minipage}%
	\begin{minipage}{0.48\columnwidth}
		\includegraphics[width=\columnwidth]{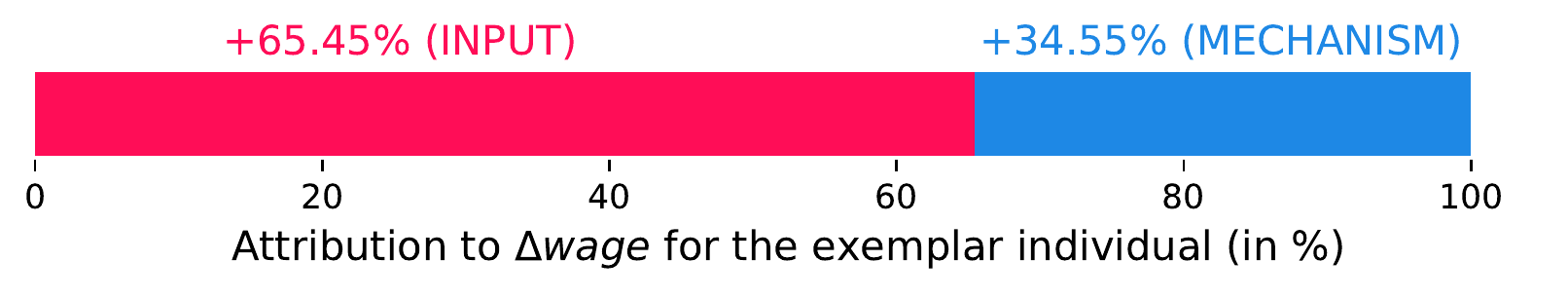}
	\end{minipage}
    \begin{minipage}{0.48\columnwidth}
        \includegraphics[width=\columnwidth]{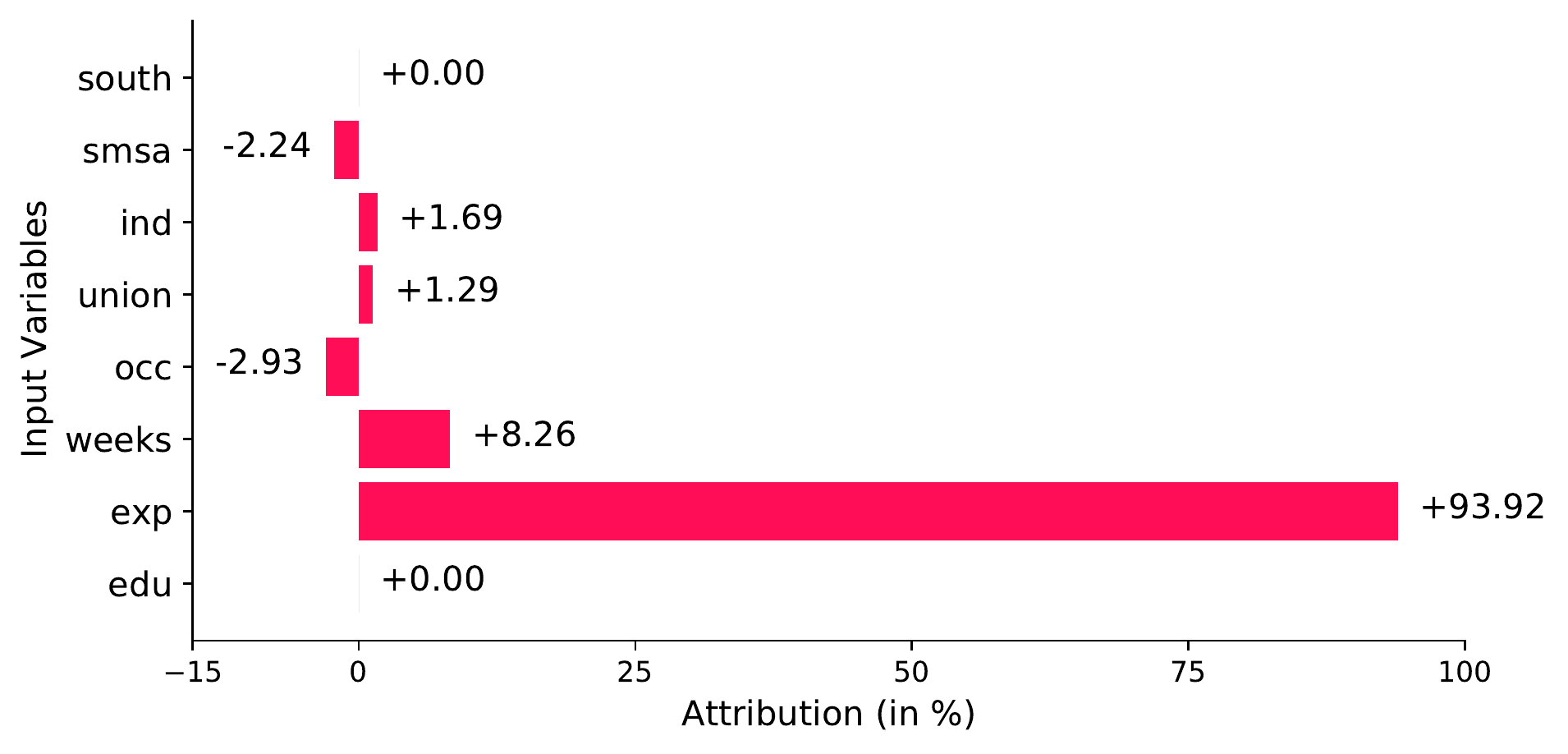}
        \subcaption{Aggregated attributions over units}
    \end{minipage}%
	\begin{minipage}{0.48\columnwidth}
		\includegraphics[width=\columnwidth]{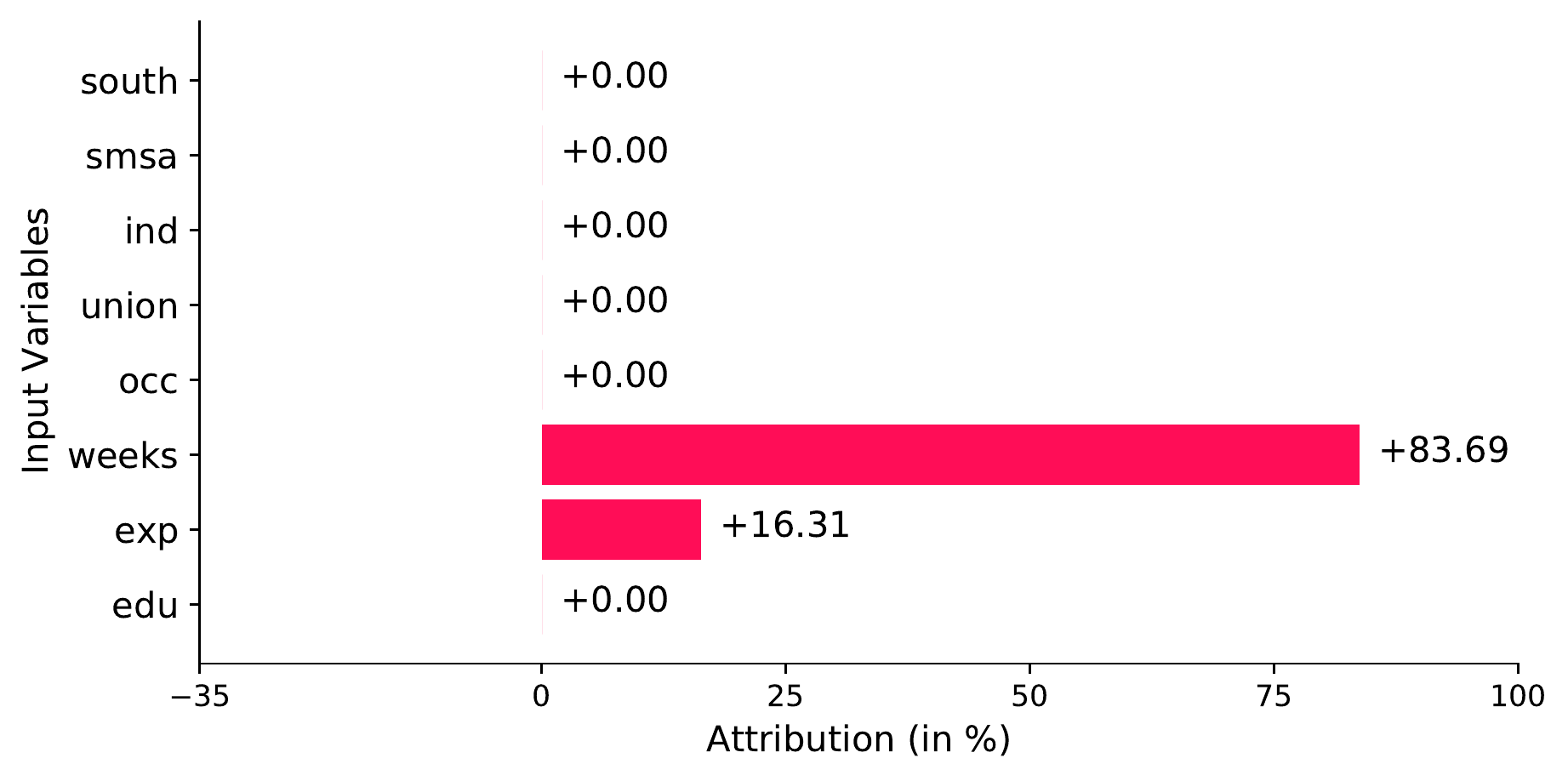}
		\subcaption{Unit-level attribution anecdote}
	\end{minipage}
    \caption{\textbf{(a)} Results showing the drivers of change in earnings reported in surveys from 1976 and 1982. The biggest driver of the change is the change in mechanism. If we breakdown the attribution to input variables further, the change in the years of full-time work experience (``exp'') is the biggest driver amongst all input variables. \textbf{(b)} Drivers of change in earnings identified by our method for an individual with the largest increase.}
    \label{fig:attribution-results}
\end{figure}
The attribution to mechanism change is reasonable for the following reason. From the model fit plot in Fig.~\ref{fig:1976-vs-1982-fit}, we observe that the coefficient of ``occ'' (white-collar occupation or not), in particular, decreased by roughly 40\%, from 0.43 in 1976 to 0.26 in 1982. The 70s were a tumultuous time in the US as the country was transitioning from a manufacturing to a service-based economy, manufacturing plants shut down, jobs were lost, and workers protested. Blue-collar jobs were replaced by white-collar jobs. But many blue-collar workers who got white-collar jobs still ended up in low paid white-collar occupations~\citep{westcott:1982:70s-jobs}. Thus, being a white-collar worker did not impact earnings greatly in the beginning of the 80s.
The attribution to input ``exp'' change is also plausible given that there is a 6 year gap between the two surveys. It is also typical to get paid better with more years of experience. We also rightly attribute 0\% to ``edu'' as education of the individuals did not change between the two surveys.

\paragraph{\textbf{Results: Unit--level attribution anecdote.}} To show how our method can be used to explain unit-level changes, we pick one anecdote from the data. In particular, we study an individual (id 167) whose reported earnings has increased the most between 1976 and 1982 (by roughly 34\%) as shown in Table~\ref{tab:unit-data} in the Appendix. Only two input values have changed for this individual, namely that of ``exp'' and ``weeks''. Our method attributes changes in input variables as the biggest driver (see Fig.~\ref{fig:attribution-results} (b)). If we further break down the attribution to input, we observe that the drivers amongst input variables are the change in the number of weeks worked (``weeks'') and the work experience (``exp''). Given the gain in experience (by 6 years) and more weeks of work reported in 1982 subject to the level of education (Master's Degree) and white-collar job, the increase in earnings is reasonable. We correctly identify the drivers that drove the increase in earnings for this individual. 



%
	\section{Discussion}\label{section:discussion}
We proposed two methods based on counterfactuals for attributing the change in output values for a statistical unit at various granularities using Shapley values from game theory. 
The first method (\CoarseAny) attributes the output change to the change in the mechanism and the change in the input vector. For linear mechanisms, we obtain a closed-form solution that can also attribute to the change in the value of each input variable. To obtain this fine--grained attribution for non--linear mechanisms, we proposed \FineAny that generalises \CoarseAny to each input variable. Both methods attribute a non-zero output change to a cause only if its value changes, and their attributions sum to the output change. As it is important for real-world applications, we studied the reliability of \CoarseAny and \FineAny when we do not know the underlying mechanisms. When the underlying mechanisms are linear, both methods are accurate regardless of which model we pick. For non-linear ground truth, the variances in attributions are larger than that for the linear ground truth. We also showed that \FineAny scales w.r.t. a number of input variables if we are willing to trade-off accuracy by applying sampling approximation to Shapley values. Lastly, we presented a case study identifying the drivers of average change in the earnings of individuals between two years, and discussed the drivers for one exemplar individual.

\textbf{On Symmetrization.} Although using the concept of Shapley values, we symmetrized over all ways of replacing causes from their background to foreground values, in practice, one may prefer a natural ordering of replacements, if available. In such cases, symmetrization is not necessary; it suffices to compute only the marginal contribution of a cause given the replacements of causes before following the natural ordering.

\textbf{On Uncertainty Quantification.} If we believe that the mechanisms have no uncertainties associated with them (e.g., among models in a class), then we get point estimates with zero uncertainties from these methods. But if we learn the mechanisms from data, uncertainties are inherent. Then it might be desirable to quantify the uncertainties of the attributions. To this end, we can compute the bootstrap confidence intervals for estimated attributions---by learning mechanisms from random subsets of data and computing the attributions using them.

\textbf{Paradox 1/ Destructive Changes.}
We note that paradoxical results are possible in special cases where the change in the mechanism and the change in the input render no change in the output. Consider the following linear mechanisms, with a single input variable $X$, where the coefficient flips its sign from positive to negative going from the background to the foreground scenario:
\[
f^{(1)}(X^{(1)}) = +X^{(1)},\quad
f^{(2)}(X^{(2)}) = -X^{(2)}
\]
Now suppose we have the following input values for a unit, whose absolute values are the same, but sign flips similarly as in case of the mechanism:
\[
x^{(1)} = +1,\quad
x^{(2)} = -1
\]
The output does not change for this unit, i.e., $\Delta y=f^{(2)}(x^{(2)})-f^{(1)}(x^{(1)})=1-1=0$. But both the mechanism and the input have changed. As such, one might expect both the mechanism and the input to get non-zero attributions. Is this the case though? Let us compute the attributions using the closed-form solutions for the linear mechanisms (\cref{sec:fine-linear}):
\begin{align*}
	\pi(f) &= \frac{1}{2} \left(-1-1\right) + \frac{1}{2} \left(1+1\right) = 0, \quad 
	\pi(x) = \frac{1}{2} \left(-1 -1 \right) + \frac{1}{2} \left( 1 +1 \right) = 0 .
\end{align*}
Contrary to our expectation, both causal drivers get zero attributions. As contributions can be negative, and we take the average of contributions over all orderings, such scenarios seem almost unavoidable. While such scenarios are possible, they are contrived. We present another paradox in the Appendix. 
	
	\bibliographystyle{icml2022}
	\bibliography{paper,literature_kaibud}

\begin{thebibliography}{27}
\providecommand{\natexlab}[1]{#1}
\providecommand{\url}[1]{\texttt{#1}}
\expandafter\ifx\csname urlstyle\endcsname\relax
  \providecommand{\doi}[1]{doi: #1}\else
  \providecommand{\doi}{doi: \begingroup \urlstyle{rm}\Url}\fi

\bibitem[Balke \& Pearl(1994)Balke and Pearl]{balke:1994:response-function}
Balke, A. and Pearl, J.
\newblock Counterfactual probabilities: Computational methods, bounds and
  applications.
\newblock In \emph{Proceedings of the Tenth International Conference on
  Uncertainty in Artificial Intelligence}, UAI'94, pp.\  46–54, San
  Francisco, CA, USA, 1994. Morgan Kaufmann Publishers Inc.

\bibitem[Budhathoki et~al.(2021)Budhathoki, Janzing, Bloebaum, and
  Ng]{budhathoki:2021:distribution-change}
Budhathoki, K., Janzing, D., Bloebaum, P., and Ng, H.
\newblock Why did the distribution change?
\newblock In \emph{Proceedings of The 24th International Conference on
  Artificial Intelligence and Statistics}, volume 130 of \emph{Proceedings of
  Machine Learning Research}, pp.\  1666--1674. PMLR, 13--15 Apr 2021.

\bibitem[Cohen \& Howe(1988)Cohen and Howe]{cohen:1988:xai-evaluating-ai}
Cohen, P.~R. and Howe, A.~E.
\newblock How evaluation guides ai research: The message still counts more than
  the medium.
\newblock \emph{AI Magazine}, 9\penalty0 (4):\penalty0 35, Dec. 1988.

\bibitem[Frye et~al.(2020)Frye, Rowat, and
  Feige]{frye:2020:xai-asymmetric-shapley}
Frye, C., Rowat, C., and Feige, I.
\newblock Asymmetric shapley values: incorporating causal knowledge into
  model-agnostic explainability.
\newblock In \emph{Advances in Neural Information Processing Systems 33: Annual
  Conference on Neural Information Processing Systems 2020, NeurIPS 2020,
  December 6-12, 2020, virtual}, 2020.

\bibitem[Goyal et~al.(2019)Goyal, Wu, Ernst, Batra, Parikh, and
  Lee]{goyal:2019:xai-counterfactual-visual-explanations}
Goyal, Y., Wu, Z., Ernst, J., Batra, D., Parikh, D., and Lee, S.
\newblock Counterfactual visual explanations.
\newblock In \emph{Proceedings of the 36th International Conference on Machine
  Learning}, volume~97 of \emph{Proceedings of Machine Learning Research}, pp.\
   2376--2384. PMLR, 09--15 Jun 2019.

\bibitem[Heskes et~al.(2020)Heskes, Sijben, Bucur, and
  Claassen]{heskes:2020:xai-causal-shap}
Heskes, T., Sijben, E., Bucur, I.~G., and Claassen, T.
\newblock Causal shapley values: Exploiting causal knowledge to explain
  individual predictions of complex models.
\newblock In \emph{Advances in Neural Information Processing Systems},
  volume~33, pp.\  4778--4789. Curran Associates, Inc., 2020.

\bibitem[Janzing et~al.(2020)Janzing, Minorics, and
  Bloebaum]{janzing:2020:feature}
Janzing, D., Minorics, L., and Bloebaum, P.
\newblock Feature relevance quantification in explainable ai: A causal problem.
\newblock In \emph{Proceedings of the Twenty Third International Conference on
  Artificial Intelligence and Statistics}, volume 108 of \emph{Proceedings of
  Machine Learning Research}, pp.\  2907--2916, Online, 26--28 Aug 2020. PMLR.

\bibitem[Johnson et~al.(2018)Johnson, McGonagle, Freedman, and
  Sastry]{johnson:2018:psid}
Johnson, D.~S., McGonagle, K.~A., Freedman, V.~A., and Sastry, N.
\newblock Fifty years of the panel study of income dynamics: Past, present, and
  future.
\newblock \emph{The ANNALS of the American Academy of Political and Social
  Science}, 680\penalty0 (1):\penalty0 9--28, 2018.

\bibitem[Kitagawa(1955)]{kitagawa:1955:xai-kitagawa-blinder-oaxaca-decomposition}
Kitagawa, E.~M.
\newblock Components of a difference between two rates*.
\newblock \emph{Journal of the American Statistical Association}, 50\penalty0
  (272):\penalty0 1168--1194, 1955.

\bibitem[Koh \& Liang(2017)Koh and Liang]{koh:2017:xai-influence-function}
Koh, P.~W. and Liang, P.
\newblock Understanding black-box predictions via influence functions.
\newblock In \emph{Proceedings of the 34th International Conference on Machine
  Learning - Volume 70}, ICML'17, pp.\  1885--1894. JMLR.org, 2017.

\bibitem[Lakkaraju et~al.(2020)Lakkaraju, Arsov, and
  Bastani]{lakkaraju:2020:xai-black-box-explain}
Lakkaraju, H., Arsov, N., and Bastani, O.
\newblock Robust and stable black box explanations.
\newblock In \emph{Proceedings of the 37th International Conference on Machine
  Learning, {ICML} 2020, 13-18 July 2020, Virtual Event}, volume 119 of
  \emph{Proceedings of Machine Learning Research}, pp.\  5628--5638. {PMLR},
  2020.

\bibitem[Lundberg \& Lee(2017{\natexlab{a}})Lundberg and
  Lee]{lundberg:2017:xai-DeepSHAP}
Lundberg, S.~M. and Lee, S.
\newblock A unified approach to interpreting model predictions.
\newblock In \emph{Advances in Neural Information Processing Systems 30: Annual
  Conference on Neural Information Processing Systems 2017, December 4-9, 2017,
  Long Beach, CA, {USA}}, pp.\  4765--4774, 2017{\natexlab{a}}.

\bibitem[Lundberg \& Lee(2017{\natexlab{b}})Lundberg and
  Lee]{lundberg:2017:shap}
Lundberg, S.~M. and Lee, S.-I.
\newblock A unified approach to interpreting model predictions.
\newblock In \emph{Advances in Neural Information Processing Systems},
  volume~30. Curran Associates, Inc., 2017{\natexlab{b}}.

\bibitem[Lundberg \& Lee(2017{\natexlab{c}})Lundberg and
  Lee]{lundberg:2017:xai-shap}
Lundberg, S.~M. and Lee, S.-I.
\newblock A unified approach to interpreting model predictions.
\newblock In \emph{Advances in Neural Information Processing Systems},
  volume~30. Curran Associates, Inc., 2017{\natexlab{c}}.

\bibitem[Newell \& Simon(1976)Newell and Simon]{newell:1976:xai-ablation-study}
Newell, A. and Simon, H.~A.
\newblock Computer science as empirical inquiry: Symbols and search.
\newblock \emph{Commun. ACM}, 19\penalty0 (3):\penalty0 113--126, mar 1976.

\bibitem[Nguyen et~al.(2019)Nguyen, Yosinski, and
  Clune]{nguyen:2019:xai-activation-maximisation}
Nguyen, A., Yosinski, J., and Clune, J.
\newblock Understanding neural networks via feature visualization: {A} survey.
\newblock In \emph{Explainable {AI:} Interpreting, Explaining and Visualizing
  Deep Learning}, volume 11700 of \emph{Lecture Notes in Computer Science},
  pp.\  55--76. Springer, 2019.

\bibitem[Pearl(2009)]{pearl:2009:book}
Pearl, J.
\newblock \emph{Causality: Models, Reasoning and Inference}.
\newblock Cambridge University Press, New York, NY, USA, 2nd edition, 2009.

\bibitem[Ribeiro et~al.(2016)Ribeiro, Singh, and
  Guestrin]{ribeiro:2016:xai-lime}
Ribeiro, M.~T., Singh, S., and Guestrin, C.
\newblock "why should i trust you?": Explaining the predictions of any
  classifier.
\newblock In \emph{Proceedings of the 22nd ACM SIGKDD International Conference
  on Knowledge Discovery and Data Mining}, pp.\  1135–1144, New York, NY,
  USA, 2016. Association for Computing Machinery.

\bibitem[Ribeiro et~al.(2018)Ribeiro, Singh, and
  Guestrin]{ribeiro:2018:xai-anchors}
Ribeiro, M.~T., Singh, S., and Guestrin, C.
\newblock Anchors: High-precision model-agnostic explanations.
\newblock \emph{Proceedings of the AAAI Conference on Artificial Intelligence},
  32\penalty0 (1), 2018.

\bibitem[Shapley(1953)]{shapley:1953:solution}
Shapley, L.~S.
\newblock A value for n-person games.
\newblock Technical report, Rand Corporation, 1953.

\bibitem[Shrikumar et~al.(2017)Shrikumar, Greenside, and
  Kundaje]{shrikumar:2017:xai-DeepLIFT}
Shrikumar, A., Greenside, P., and Kundaje, A.
\newblock Learning important features through propagating activation
  differences.
\newblock In \emph{Proceedings of the 34th International Conference on Machine
  Learning - Volume 70}, ICML'17, pp.\  3145--–3153. JMLR.org, 2017.

\bibitem[Singal et~al.(2021)Singal, Michailidis, and Ng]{singal:2021:xai-rsv}
Singal, R., Michailidis, G., and Ng, H.
\newblock Flow-based attribution in graphical models: A recursive shapley
  approach.
\newblock In \emph{Proceedings of the 38th International Conference on Machine
  Learning}, volume 139 of \emph{Proceedings of Machine Learning Research},
  pp.\  9733--9743. PMLR, 18--24 Jul 2021.

\bibitem[Strumbelj \& Kononenko(2014)Strumbelj and
  Kononenko]{kononenko:2014:sampling-shap}
Strumbelj, E. and Kononenko, I.
\newblock Explaining prediction models and individual predictions with feature
  contributions.
\newblock \emph{Knowl. Inf. Syst.}, 41\penalty0 (3):\penalty0 647--–665,
  2014.

\bibitem[Sundararajan \& Najmi(2020)Sundararajan and
  Najmi]{sundararajan:2020:xai-BShap}
Sundararajan, M. and Najmi, A.
\newblock The many shapley values for model explanation.
\newblock In \emph{Proceedings of the 37th International Conference on Machine
  Learning}, volume 119 of \emph{Proceedings of Machine Learning Research},
  pp.\  9269--9278. PMLR, 13--18 Jul 2020.

\bibitem[Sundararajan et~al.(2017)Sundararajan, Taly, and
  Yan]{sundararajan:2017:xai-axioms}
Sundararajan, M., Taly, A., and Yan, Q.
\newblock Axiomatic attribution for deep networks.
\newblock In Precup, D. and Teh, Y.~W. (eds.), \emph{Proceedings of the 34th
  International Conference on Machine Learning}, volume~70 of \emph{Proceedings
  of Machine Learning Research}, pp.\  3319--3328. PMLR, 06--11 Aug 2017.

\bibitem[Westcott(1982)]{westcott:1982:70s-jobs}
Westcott, D.~N.
\newblock Blacks in the 1970's: Did they scale the job ladder?.
\newblock \emph{Monthly Labor Review}, 105:\penalty0 29--38, 1982.

\bibitem[Zhang et~al.(2015)Zhang, Wang, Zhang, and
  Sch\"{o}lkopf]{zhang:2015:invertible-fcm}
Zhang, K., Wang, Z., Zhang, J., and Sch\"{o}lkopf, B.
\newblock On estimation of functional causal models: General results and
  application to the post-nonlinear causal model.
\newblock \emph{ACM Trans. Intell. Syst. Technol.}, 7\penalty0 (2), 2015.

\end{thebibliography}
	
	\newpage
	\appendix
	\section{Experiment Details}
The data from the Panel Study of Income Dynamics is available online as the dataset ``PSID7682'' from the R package ``\href{https://cran.r-project.org/web/packages/AER}{AER}'' with the GPLv3 license. We describe the variables used for this case study in Table~\ref{tab:wage-summary}. From the original data, we exclude sensitive attributes like sex and race, and irrelevant attributes like marital status. Although they do not affect our analysis, learning fair models is not the main point of this case study. In Fig.~\ref{fig:1976-vs-1982-jointplot}, we show the scatter plot between earnings reported in those two surveys, along with their univariate histograms. 

\begin{table}[tb]
	\centering
	\renewcommand{\arraystretch}{1.2}
	\caption{Description of the variables in the earnings data.}
	\label{tab:wage-summary}
	\resizebox{\columnwidth}{!}{%
		\begin{tabular}{rll}
				\toprule
				Variable & Description & Type\\
				\midrule
				wage & Natural logarithm of earnings in past year & target\\
				edu & Years of education & input\\
				exp & Years of full-time work experience & input\\
				occ & Is the individual a white-collar worker? & input\\
				weeks & Weeks worked past year & input\\
				union & Is the individual a member of a union? & input\\
				ind & Does the individual work in a manufacturing industry? & input\\
				smsa & Does the individual reside in a standard metropolitan statistical area? & input\\
				south & Does the individual live in the South? & input\\
				\bottomrule
		\end{tabular}
	}
\end{table}

\begin{figure}[tb]
	\centering
	\includegraphics[width=0.5\columnwidth]{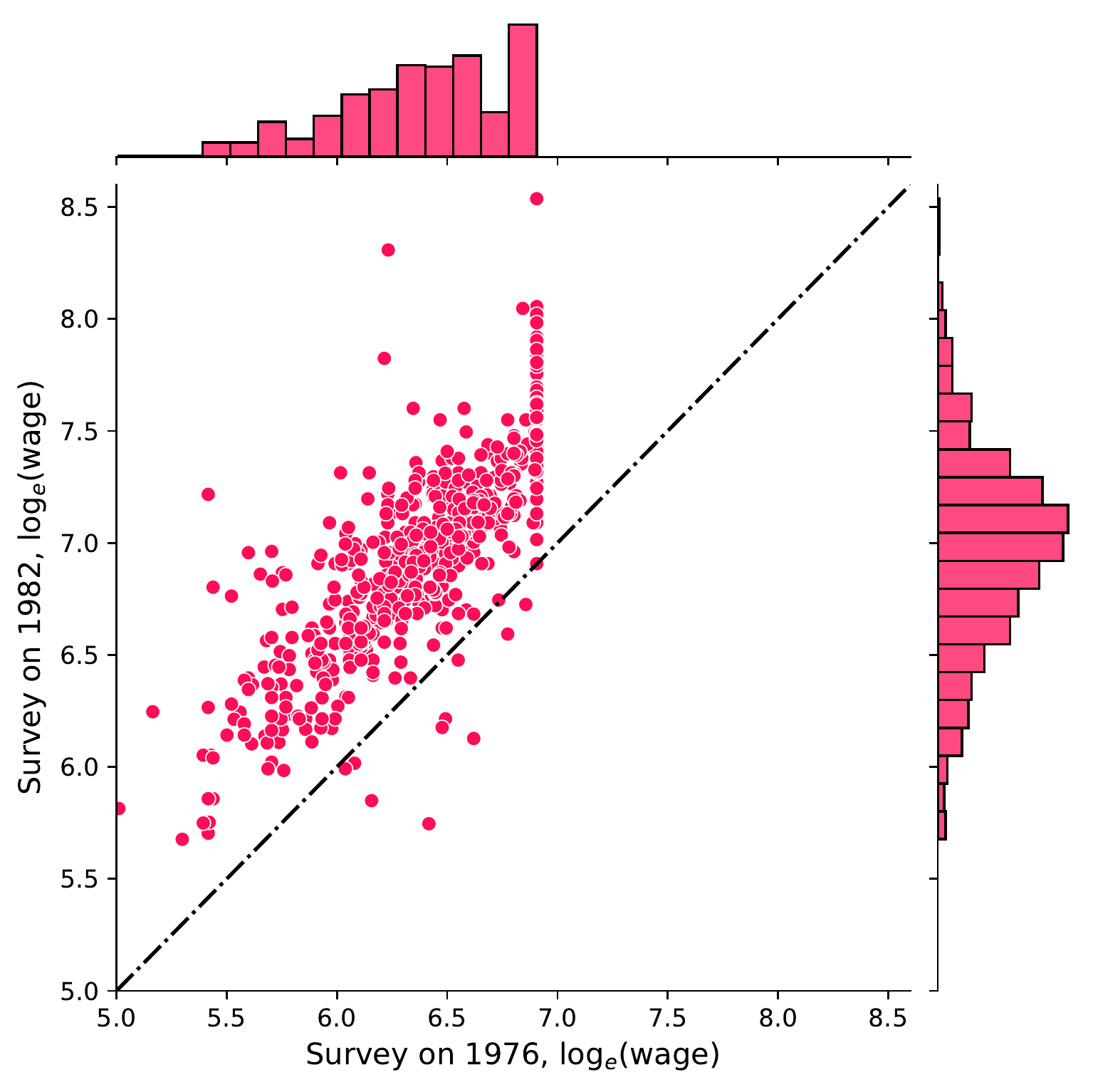}
	\caption{Natural logarithm of earnings from the past year from the survey conducted in 1976 versus that in 1982. Each dot represents the head individual of a household in the survey. For roughly 98\% of individuals, earnings have increased in the 1982 survey.}
\label{fig:1976-vs-1982-jointplot}
\end{figure}

\begin{figure}[H]
	\centering
	\includegraphics[width=0.7\columnwidth]{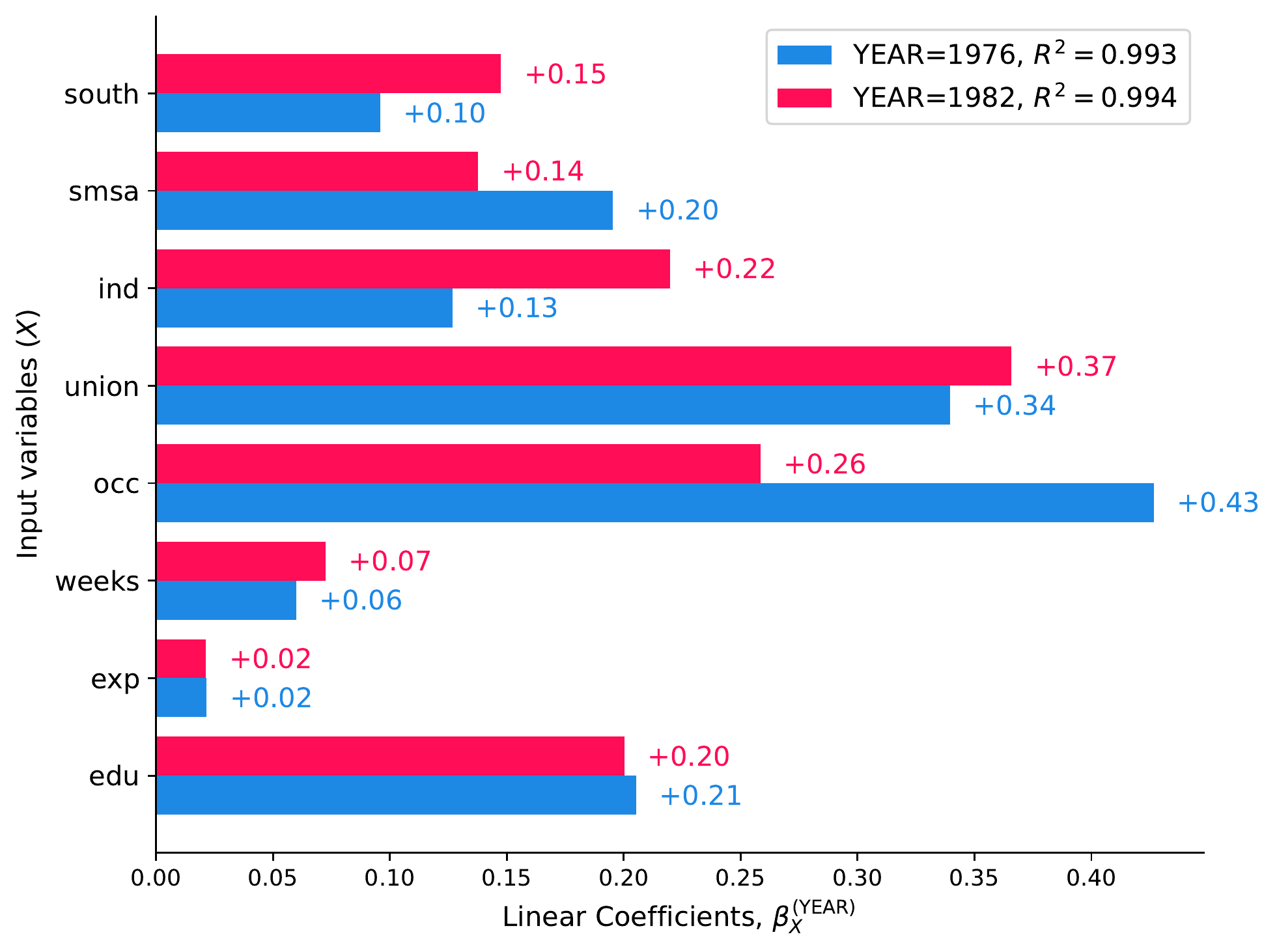}
	\caption{Linear model fits and diagnostics on data from the surveys conducted in 1976 and 1982. The R-squared fit of linear models in both data indicate that linear models are, indeed, appropriate for this case study.}
	\label{fig:1976-vs-1982-fit}
\end{figure}

\begin{table}[H]
	\centering
	\caption{Individual with largest increase in earnings.}
	\label{tab:unit-data}
		\renewcommand{\arraystretch}{1.3}
		\begin{tabular}{l | r r r r r r r r | r}
			\toprule
			year & edu & exp & weeks & occ & union & ind & smsa & south & wage \\
			\midrule
			1976 & 17.0 & \cellcolor{lightgray!30} 3.0 & \cellcolor{lightgray!30} 40.0 & 1.0 & 0.0 & 0.0 & 1.0 & 0.0 & \textbf{6.2}\\
			1982 & 17.0 & \cellcolor{lightgray!30} 9.0 & \cellcolor{lightgray!30} 50.0 & 1.0 & 0.0 & 0.0 & 1.0 & 0.0 & \textbf{8.3}\\
			\bottomrule
		\end{tabular}
\end{table}

\section{Unit-level changes with stochastic causal mechanisms}
So far, we considered deterministic algorithmic mechanisms $f$ that output $Y$ from inputs $\bfX$. The inputs $\bfX$ are the \emph{causes} of $Y$ w.r.t. the algorithm $f$. But they are not necessarily the causal ancestors of output $Y$ in the real world.  If our objective is to explain the change in the real output value $\Delta y$ accounting for causal mechanisms in real-world, we need to consider the causal relationships between variables $X_1, \dotsc, X_d, Y$ and their causal mechanisms.

Suppose that the ``mechanisms'' by which the output $Y$ is generated from its causes $X$ in the real world is stochastic. In particular, we make three assumptions here:

\textbf{Assumption C1/} We have access to the causal graph of variables $X_1, \dotsc, X_d, Y$.

\textbf{Assumption C2/} We have access to the associated functional causal model (FCM)~\citep{pearl:2009:book}. 

\textbf{Assumption C3/} The FCM is invertible~\citep{zhang:2015:invertible-fcm}.

Without loss of generality, let $Y \eqqcolon X_{d+1}$ be the sink node, i.e., $Y$ does not have any children. In an FCM, each variable $X_j$ is a function of its parents $\mathrm{PA}_j$ in the causal graph and an unobserved noise variable $N_j$, i.e., 
\begin{align}
	X_j \coloneqq f_j(\mathrm{PA}_j, N_j),\label{eq:fcm}
\end{align}
where $N_1, \dotsc, N_{d+1}$ are independent~\citep{pearl:2009:book}. The root nodes do not have any observed parents, only unobserved noise terms. Therefore, the stochastic properties of observables $X_1, \dotsc, X_d, X_{d+1}$ are derived from the joint distribution of unobserved noise variables $P_{N_1, \dotsc, N_{d+1}} = P_{N_1} \times \dotsm \times P_{N_{d+1}}$.  
By recursively applying Eq.~\eqref{eq:fcm} until we reach the root nodes (which do not have any observed parents), we can write $Y$ as a deterministic function of all noise variables $N \coloneqq (N_1, \dotsc, N_{d+1})$ instead of observed variables $X$, i.e.,
\begin{align}
	Y \coloneqq F(N_1, \dotsc, N_{d+1}), \label{eq:fcm-noise-only}
\end{align}
where $F:\mathcal{N}_1 \times \dotsm\times \mathcal{N}_{m+1} \to \mathcal{Y}$ with $\mathcal{N}_j$ being the space of values of noise variable $N_j$. In other words, the unobserved variables $N_1, \dotsc, N_{d+1}$ determine the observed output $Y$. But the role of the composite function $F$ is rather subtle as it is tied to the noise variables. We can make it more explicit through the response-function based formulation of FCMs~\citep{balke:1994:response-function}.

To motivate the response-function based formulation, let us consider a simple case where the causal graph is $X \to Y$. The structural equation of $Y$ is a stochastic function: $Y \coloneqq f(X, N)$, which reduces to a deterministic function of $X$ for a fixed value $n$ of noise $N$: $Y \coloneqq f(X, n)$. That is, if $X$ and $Y$ take values in $\calX$ and $\calY$ respectively, then noise $N$ acts as a random switch that selects different functions from $\calX$ to $\calY$. Without loss of generality, we can therefore assume that $N$ takes values in the set of functions from $\calX$ to $\calY$, denoted by $\calY^\calX$. Then we can rewrite the structural equation of $Y$ as $Y \coloneqq N(X)$. The FCM $Y \coloneqq f(X,N)$ has now turned into a probability distribution $P_N$ on the set of deterministic functions $\calY^\calX$. For example, if $X$ and $Y$ are binary, i.e., $\calX \coloneqq \{0, 1\}$ and $Y \coloneqq \{0, 1\}$, then there are four possible functions from $\calX$ to $\calY$, i.e., $\calY^\calX = \{\mathbf{0}, \mathbf{1}, \texttt{ID}, \texttt{NOT}\}$, where $\mathbf{0}$ and $\mathbf{1}$ denote constant functions that always map to 0 and 1 respectively, and \texttt{ID} and \texttt{NOT} denote identity and negation respectively. Note that as root node $X$ does not have any observed parents, its noise value selects a trivial identity function that simply maps the noise.

Generalising this idea to $d+1$ variables, we obtain the value $y$ of $Y$ (i.e., $X_{d+1}$) through a series of transformations at each causal ancestor, starting from root nodes following the causal ordering: 
\begin{align}
	y \coloneqq n_{d+1}\Big(n_{\mathrm{PA}_{Y1}}\big( \dotso \big), \dotsc, n_{\mathrm{PA}_{Ym}}\big( \dotso \big)\Big),
\end{align}
where $\mathrm{PA}_{Yi}$ is the $i$-th parent of $Y$, and lower case $n_j$ represents the value of noise $N_j$ corresponding to the observable $X_j$. That is, the value of $y$ is determined by the functions selected by noise values $n_1, \dotsc, n_{d+1}$. Since the FCM is invertible, we can uniquely recover the unobserved noise values $n_1, \dotsc, n_{d+1}$ from the observed values $x_1, \dotsc, x_{d+1}$, where $x_{d+1} \coloneqq y$. Therefore, the candidates for ``root causes'' of the output change $\Delta y \coloneqq y^{(2)} - y^{(1)}$ are the deterministic mechanisms at each node $X_j$ selected by the noise values. 

Let $R \coloneqq \{n_1, \dotsc, n_{d+1}\}$ denote the set of causes of $\Delta y$.
To attribute $\Delta y$ to its ``root causes'', we can define the contribution of a cause $r \in R$ as in \cref{sec:fine-any}: quantify how much $y$ would change if we change the noise value corresponding to the cause to its foreground value (i.e., replace the background mechanism at $r$ by the foreground mechanism) given that we have already changed background mechanisms in $A \subset R$ to foreground mechanisms:
\begin{align*}
	C(r \mid A; R) \coloneqq n_{d+1}^{\left(\mathbf{2}_{A \cup \{r\}}(n_{d+1})\right)}\Big(n_{\mathrm{PA}_{Y1}}^{\left(\mathbf{2}_{A \cup \{r\}}(n_{\mathrm{PA}_{Y1}})\right)}\big( \dotso \big), \dotsc \Big) - n_{d+1}^{\left(\mathbf{2}_{A}(n_{d+1})\right)}\Big(n_{\mathrm{PA}_{Y1}}^{\left(\mathbf{2}_{A}(\mathrm{PA}_{Y1})\right)}\big( \dotso \big), \dotsc \Big)
\end{align*}
To avoid arbitrariness introduced by the order in which we replace mechanisms, we symmetrize by taking the average of contributions over all orderings of $R$ to get the Shapley value contribution. This technique operationalises the proposal for attributing distributional changes to causal mechanisms of nodes in the causal graph~\citep{budhathoki:2021:distribution-change} to unit-level.

%

\section{Paradox 2/ Non-additivity for multiple changes when collapsing changes}
As systems may undergo change multiple times, one may wonder whether the attributions add up when we "collapse" changes. More formally, given three time points $t_1, t_2, t_3$, does the sum of the attributions for two consecutive changes (i.e., $t_1 \to t_2$ and $t_2 \to t_3$) equal the attributions for a single change obtained by collapsing the two changes (i.e., $t_1 \to t_3$)? 

Consider a mechanism $f$ connecting a single input variable $X$ to the output $Y$. Let $x^{(k)}$ denote the value of $X$ for a unit from the time point $t_k$. Likewise, define $y^{(k)}$ and $f^{(k)}$. Clearly the output changes are additive when we collapse changes, i.e., $\left(y^{(2)} - y^{(1)} \right) + \left( y^{(3)} - y^{(2)} \right) = y^{(3)} - y^{(1)}$. What about the attributions? As we have a single input variable, both \FineAny and \CoarseAny yield the same Shapley values. In particular, for the change from $t_1$ to $t_2$, the Shapley value for the mechanism $f$ is given by
\begin{align*}
	\pi^{1\to2}({f}) &\coloneqq \frac{1}{2}\left\{ f^{(2)}\big(x^{(1)}\big) - f^{(1)}\big(x^{(1)}\big) \right\} 
	+
	\frac{1}{2}\left\{ f^{(2)}\big(x^{(2)}\big) - f^{(1)}\big(x^{(2)}\big) \right\}, 
\end{align*}
and that for the change from $t_2$ to $t_3$ is given by
\begin{align*}
	\pi^{2\to3}({f}) &\coloneqq \frac{1}{2}\left\{ f^{(3)}\big(x^{(2)}\big) - f^{(2)}\big(x^{(2)}\big) \right\} 
	+
	\frac{1}{2}\left\{ f^{(3)}\big(x^{(3)}\big) - f^{(2)}\big(x^{(3)}\big) \right\}.
\end{align*}
Their sum is given by
\begin{align*}
	\pi^{1\to2}({f}) + \pi^{2\to3}({f}) &\coloneqq \frac{1}{2}\left\{ f^{(2)}\big(x^{(1)}\big) - f^{(1)}\big(x^{(1)}\big) \right\} 
	+
	\frac{1}{2}\left\{ f^{(3)}\big(x^{(2)}\big) - f^{(1)}\big(x^{(2)}\big) \right\}
	+\\&\phantom{afc} 
	\frac{1}{2}\left\{ f^{(3)}\big(x^{(3)}\big) - f^{(2)}\big(x^{(3)}\big) \right\}\\
	&\neq \pi^{1\to3}({f}),
\end{align*}
where $\pi^{1\to3}({f})$ is given by
\begin{align*}
	\pi^{1\to3}({f}) &\coloneqq \frac{1}{2}\left\{ f^{(3)}\big(x^{(1)}\big) - f^{(1)}\big(x^{(1)}\big) \right\} 
	+
	\frac{1}{2}\left\{ f^{(3)}\big(x^{(3)}\big) - f^{(1)}\big(x^{(3)}\big) \right\}.
\end{align*}
This shows that attributions are non-additive when we collapse the intermediate changes. This phenomenon can be explained by the inclusion of counterfactuals in our attribution methods. When we attribute a single change (i.e. $t_1 \to t_3$) obtained by the collapsing the intermediate changes ($t_1 \to t_2$ and $t_2 \to t_3$), we do not consider the counterfactuals involving intermediate states any more (e.g., $f^{(3)}(x^{(2)})$, $f^{(1)}(x^{(2)})$). This conclusion also holds for the Shapley values of the input $X$, following similar algebraic calculation. 

\section{Proofs}
\begin{theorem}
Both attribution methods satisfy completeness and dummy axioms.
\end{theorem}
\begin{proof}
	The completeness property follows directly from the property of Shapley values~\citep{shapley:1953:solution}. That is, the Shapley values of players in a ``coalition game'' sum up to the value of the game. In our case, the players are the causal drivers of $\Delta y$. In all the attribution techniques, we start with background values only (e.g., $f^{(1)}(\bfx^{(1)})$ or $f^{(1)}(x_1^{(1)}, \dotsc, x_d^{(1)})$)---the worth of an empty coalition of players. As we gradually replace the causal drivers to their foreground values, we end up with foreground values only (e.g., $f^{(2)}(\bfx^{(2)})$ or or $f^{(2)}(x_1^{(2)}, \dotsc, x_d^{(2)})$)---the worth of the grand coalition where all players join. Therefore, the players add a total value of $f^{(2)}(\bfx^{(2)})-f^{(1)}(\bfx^{(1)}) = \Delta y$.
	
	We show the dummy property for the coarse-grained attribution first. The causal drivers of $\Delta y$ are in the set $W \coloneqq \{f, \bfx\}$. Let us recap the contribution of a causal driver $w$ given that we have replaced the causal drivers in the subset $A \subset W$ to their foreground values:
	\begin{align*}
		C(w \mid A; W) &\coloneqq f^{\left(\mathbf{2}_{A \cup \{w\}}(f)\right)} \left(\bfx^{\left(\mathbf{2}_{A \cup \{w\}}(\bfx)\right)}\right) - 
		f^{\left(\mathbf{2}_{A}(f)\right)} \left(\bfx^{\left(\mathbf{2}_{A}(\bfx)\right)}\right).
	\end{align*}
	Suppose that the mechanism $f$ does not change, i.e. $f^{(2)}(\bfx) = f^{(1)}(\bfx)$ for all $\bfx \in \calbfX$. Then the contribution of the mechanism $f$ to $\Delta y$ given \emph{any} subset $A$ reduces to
	\begin{align*}
		C(f \mid A; W) 
		&= f^{\left(\mathbf{2}_{A \cup \{f\}}(f)\right)} \left(\bfx^{\left(\mathbf{2}_{A \cup \{f\}}(\bfx)\right)}\right) - f^{\left(\mathbf{2}_{A}(f)\right)} \left(\bfx^{\left(\mathbf{2}_{A}(\bfx)\right)}\right)\\
		&= f^{(2)} \left(\bfx^{\left(\mathbf{2}_{A \cup \{f\}}(\bfx)\right)}\right) - 
		f^{(1)} \left(\bfx^{\left(\mathbf{2}_{A}(\bfx)\right)}\right)\\
		&= f^{(2)} \left(\bfx^{\left(\mathbf{2}_{A}(\bfx)\right)}\right) - 
		f^{(1)} \left(\bfx^{\left(\mathbf{2}_{A}(\bfx)\right)}\right)\\
		&= f^{(1)} \left(\bfx^{\left(\mathbf{2}_{A}(\bfx)\right)}\right) - 
		f^{(1)} \left(\bfx^{\left(\mathbf{2}_{A}(\bfx)\right)}\right)= 0,
	\end{align*}
	where indicator functions $\mathbf{2}_{A \cup \{f\}}(\bfx) = \mathbf{2}_{A}(\bfx)$ and $\mathbf{2}_{A}(f) = 1$ as any subset $A$ does not contain $f$; $A$ is either an empty set $\emptyset$ or $\{\bfx\}$. That is, the contribution of the mechanism $f$ to $\Delta y$ is zero regardless of the ordering $A$ in which we make replacements. As a result, the Shapley value of the mechanism $f$, which is an average of contributions over all orderings $A$, is also zero.
	
	Now suppose that the input vector does not change, i.e. $\bfx^{(2)} = \bfx^{(1)}$. Then the contribution of the inputs $\bfX$ to $\Delta y$ given \emph{any} subset $A$ reduces to
	\begin{align*}
		C(\bfx \mid A; W) 
		&= f^{\left(\mathbf{2}_{A \cup \{\bfx\}}(f)\right)} \left(\bfx^{\left(\mathbf{2}_{A \cup \{\bfx\}}(\bfx)\right)}\right) - f^{\left(\mathbf{2}_{A}(f)\right)} \left(\bfx^{\left(\mathbf{2}_{A}(\bfx)\right)}\right)\\
		&= f^{\left(\mathbf{2}_{A \cup \{\bfx\}}(f)\right)} \left(\bfx^{(2)}\right) - f^{\left(\mathbf{2}_{A}(f)\right)} \left(\bfx^{(1)}\right)\\
		&= f^{\left(\mathbf{2}_{A}(f)\right)} \left(\bfx^{(2)}\right) - f^{\left(\mathbf{2}_{A}(f)\right)} \left(\bfx^{(1)}\right)\\
		&= f^{\left(\mathbf{2}_{A}(f)\right)} \left(\bfx^{(1)}\right) - f^{\left(\mathbf{2}_{A}(f)\right)} \left(\bfx^{(1)}\right)=0,
	\end{align*}
	where indicator functions $\mathbf{2}_{A \cup \{\bfx\}}(f) = \mathbf{2}_{A}(f)$ and $\mathbf{2}_{A}(\bfx) = 1$ as any subset $A$ does not contain $\bfx$; $A$ is either an empty set $\emptyset$ or $\{f\}$. That is, the contribution of the input $\bfx$ to $\Delta y$ is zero regardless of the ordering $A$ in which we make replacements. As a result, the Shapley value of the input $\bfx$, which is an average of contributions over all orderings $A$, is also zero.
	
	For the fine-grained attribution, we can follow similar arguments. But for the sake of completeness, we show the proof regardless. The causal drivers of $\Delta y$ are in the set $V \coloneqq \{f, x_1, \dotsc, x_d\}$. Let us recap the contribution of a causal driver $w$ given that we have replaced the causal drivers in the subset $A \subset V$ to their foreground values:
	\begin{align*}
		C(w \mid A; V) \coloneqq 
		f^{\left(\mathbf{2}_{A \cup \{w\}}(f)\right)} \left(x_1^{\left(\mathbf{2}_{A \cup \{w\}}(x_1)\right)}, \dotsc, x_d^{\left(\mathbf{2}_{A \cup \{w\}}(x_d)\right)}\right) -
		f^{\left(\mathbf{2}_{A}(f)\right)} \left(x_1^{\left(\mathbf{2}_{A}(x_1)\right)}, \dotsc, x_d^{\left(\mathbf{2}_{A}(x_d)\right)}\right).
	\end{align*}
	Suppose that the mechanism $f$ does not change, i.e. $f^{(2)}(\bfx) = f^{(1)}(\bfx)$ for all $\bfx \in \calbfX$. Then the contribution of the mechanism $f$ to $\Delta y$ given \emph{any} subset $A$ reduces to
	\begin{align*}
		C(f \mid A; V) 
		&= f^{\left(\mathbf{2}_{A \cup \{f\}}(f)\right)} \left(x_1^{\left(\mathbf{2}_{A \cup \{f\}}(x_1)\right)}, \dotsc, x_d^{\left(\mathbf{2}_{A \cup \{f\}}(x_d)\right)}\right) - f^{\left(\mathbf{2}_{A}(f)\right)} \left(x_1^{\left(\mathbf{2}_{A}(x_1)\right)}, \dotsc, x_d^{\left(\mathbf{2}_{A}(x_d)\right)}\right)\\
		&= f^{(2)} \left(x_1^{\left(\mathbf{2}_{A \cup \{f\}}(x_1)\right)}, \dotsc, x_d^{\left(\mathbf{2}_{A \cup \{f\}}(x_d)\right)}\right) - f^{(1)} \left(x_1^{\left(\mathbf{2}_{A}(x_1)\right)}, \dotsc, x_d^{\left(\mathbf{2}_{A}(x_d)\right)}\right)\\
		&= f^{(2)} \left(x_1^{\left(\mathbf{2}_{A}(x_1)\right)}, \dotsc, x_d^{\left(\mathbf{2}_{A}(x_d)\right)}\right) - f^{(1)} \left(x_1^{\left(\mathbf{2}_{A}(x_1)\right)}, \dotsc, x_d^{\left(\mathbf{2}_{A}(x_d)\right)}\right)\\
		&= f^{(1)} \left(x_1^{\left(\mathbf{2}_{A}(x_1)\right)}, \dotsc, x_d^{\left(\mathbf{2}_{A}(x_d)\right)}\right) - f^{(1)} \left(x_1^{\left(\mathbf{2}_{A}(x_1)\right)}, \dotsc, x_d^{\left(\mathbf{2}_{A}(x_d)\right)}\right)=0,
	\end{align*}
	where indicator functions $x_1^{\left(\mathbf{2}_{A \cup \{f\}}(x_1)\right)} = x_1^{\left(\mathbf{2}_{A}(x_1)\right)}$ and $\mathbf{2}_{A}(f)=1$ as any subset $A$ does not contain $f$; $A$ is an element of the powerset $\mathcal{P}({\{x_1, \dotsc, x_d\}})$. That is, the contribution of the mechanism $f$ to $\Delta y$ is zero regardless of the ordering $A$ in which we make replacements. As a result, the Shapley value of the mechanism $f$, which is an average of contributions over all orderings $A$, is also zero.
	
	Now suppose that the value of an input variable $X_i$ does not change, i.e. $x_i^{(2)} = x_i^{(1)}$. Then the contribution of the input variable $X_i$ to $\Delta y$ given \emph{any} subset $A$ reduces to
	\begin{align*}
		C(x_i \mid A; V) 
		&= f^{\left(\mathbf{2}_{A \cup \{x_i\}}(f)\right)} \left(x_1^{\left(\mathbf{2}_{A \cup \{x_i\}}(x_1)\right)}, \dotsc, x_i^{\left(\mathbf{2}_{A \cup \{x_i\}}(x_i)\right)}, \dotsc, x_d^{\left(\mathbf{2}_{A \cup \{x_i\}}(x_d)\right)}\right) -\\ &\phantom{asd}f^{\left(\mathbf{2}_{A}(f)\right)} \left(x_1^{\left(\mathbf{2}_{A}(x_1)\right)}, \dotsc, x_j^{\left(\mathbf{2}_{A}(x_j)\right)}, \dotsc,  x_d^{\left(\mathbf{2}_{A}(x_d)\right)}\right)\\
		&= f^{\left(\mathbf{2}_{A}(f)\right)} \left(x_1^{\left(\mathbf{2}_{A}(x_1)\right)}, \dotsc, x_i^{(2)}, \dotsc, x_d^{\left(\mathbf{2}_{A}(x_d)\right)}\right) - f^{\left(\mathbf{2}_{A}(f)\right)} \left(x_1^{\left(\mathbf{2}_{A}(x_1)\right)}, \dotsc, x_i^{(1)}, \dotsc,  x_d^{\left(\mathbf{2}_{A}(x_d)\right)}\right)\\ 
		&= f^{\left(\mathbf{2}_{A}(f)\right)} \left(x_1^{\left(\mathbf{2}_{A}(x_1)\right)}, \dotsc, x_i^{(1)}, \dotsc, x_d^{\left(\mathbf{2}_{A}(x_d)\right)}\right) - f^{\left(\mathbf{2}_{A}(f)\right)} \left(x_1^{\left(\mathbf{2}_{A}(x_1)\right)}, \dotsc, x_i^{(1)}, \dotsc,  x_d^{\left(\mathbf{2}_{A}(x_d)\right)}\right)\\ 
		&= 0,
	\end{align*}
	where indicators functions $\mathbf{2}_{A \cup \{x_i\}}(f)=\mathbf{2}_{A}(f)$, $\mathbf{2}_{A \cup \{x_i\}}(x_j) = \mathbf{2}_{A}(x_j)$, and $\mathbf{2}_{A}(x_i)=1$ as any subset $A$ does not contain $x_i$; $A$ is an element of the powerset $\mathcal{P}({\{f, x_1, \dotsc, x_d\} \setminus \{x_i\}})$. That is, the contribution of the input variable $X_j$ to $\Delta y$ is zero regardless of the ordering $A$ in which we make replacements. As a result, the Shapley value of the input $X_i$, which is an average of contributions over all orderings $A$, is also zero.
\end{proof}

\begin{theorem}
		When the mechanisms are linear, i.e., $f^{(k)}(\bfx^{(k)}) \coloneqq \sum_{j=1}^d \beta_j^{(k)} x_j^{(k)}, \ k=1,2$, the Shapley value attributions of both methods are the same, i.e., $\varphi(w) = \pi(w)$ for $w \in \{f, x_1, \dotsc, x_d\}$.
\end{theorem}
\begin{proof}
	Consider linear mechanisms of the form:
	\begin{align*}
		f^{(k)}(\bfx^{(k)}) \coloneqq \sum_{j=1}^d \beta_j^{(k)} x_j^{(k)}, \ k=1,2,
	\end{align*}
	where $\beta_j^{(k)}$'s are the linear coefficients. The causal drivers are in the set $V \coloneqq \{f, x_1, \dotsc, x_d\}$.
	Let us write down the contribution of the input variable $X_i$ given a subset $A \subset V$ with \FineAny assuming the aforementioned linear mechanisms:
	\begin{align*}
		C(x_i \mid A; V) 
		&= f^{\left(\mathbf{2}_{A \cup \{x_i\}}(f)\right)} \left(x_1^{\left(\mathbf{2}_{A \cup \{x_i\}}(x_1)\right)}, \dotsc, x_i^{\left(\mathbf{2}_{A \cup \{x_i\}}(x_i)\right)}, \dotsc, x_d^{\left(\mathbf{2}_{A \cup \{x_i\}}(x_d)\right)}\right) -\\ &\phantom{asd}f^{\left(\mathbf{2}_{A}(f)\right)} \left(x_1^{\left(\mathbf{2}_{A}(x_1)\right)}, \dotsc, x_i^{\left(\mathbf{2}_{A}(x_i)\right)}, \dotsc,  x_d^{\left(\mathbf{2}_{A}(x_d)\right)}\right)\\
		&= f^{\left(\mathbf{2}_{A}(f)\right)} \left(x_1^{\left(\mathbf{2}_{A}(x_1)\right)}, \dotsc, x_i^{(2)}, \dotsc, x_d^{\left(\mathbf{2}_{A}(x_d)\right)}\right) - f^{\left(\mathbf{2}_{A}(f)\right)} \left(x_1^{\left(\mathbf{2}_{A}(x_1)\right)}, \dotsc, x_i^{(1)}, \dotsc,  x_d^{\left(\mathbf{2}_{A}(x_d)\right)}\right)\\ 
		&= \beta_i^{\left(\mathbf{2}_{A}(f)\right)} x_i^{(2)} + \sum_{\substack{j=1\\j\neq i}}^d \beta_j^{\left(\mathbf{2}_{A}(f)\right)} x_j^{\left(\mathbf{2}_{A}(x_j)\right)} - \beta_i^{\left(\mathbf{2}_{A}(f)\right)} x_i^{(1)} - \sum_{\substack{j=1\\j\neq i}}^d \beta_j^{\left(\mathbf{2}_{A}(f)\right)} x_j^{\left(\mathbf{2}_{A}(x_j)\right)}\\
		&= \beta_i^{\left(\mathbf{2}_{A}(f)\right)} \left( x_i^{(2)} -  x_i^{(1)} \right).
	\end{align*}
	In particular, the coefficient $\beta_i^{\left(\mathbf{2}_{A}(f)\right)}$ can assume either a foreground or a background value depending on whether $A$ contains $f$. As such, we have
	\begin{align*}
		C(x_i \mid A; V) = \begin{cases}
			\beta_i^{(2)} \left( x_i^{(2)} -  x_i^{(1)} \right) & \text{ when $f \in A$}\\
			\beta_i^{(1)} \left( x_i^{(2)} -  x_i^{(1)} \right) & \text{ when $f \notin A$}.
		\end{cases}
	\end{align*}
	Note that $A$ is an element of the powerset $\mathcal{P}(\{f, x_1, \dots, x_d\} \setminus \{x_i\})$ which contains $2^{d}$ elements. Each set $A$ in the powerset containing $f$ can be paired with another set that does not contain $f$. Therefore, half of the sets $A$ in the powerset contain $f$, half do not. Thus the average contribution of $X_i$, over all subsets $A$, which is also its Shapley value, is given by 
	\begin{align*}
		\varphi(x_i) &= \frac{1}{2} \beta_i^{(2)} \left( x_i^{(2)} -  x_i^{(1)} \right) + \frac{1}{2} \beta_i^{(1)} \left( x_i^{(2)} -  x_i^{(1)} \right)
		= \frac{\beta_i^{(1)}+\beta_i^{(2)}}{2} \left( x_i^{(2)} -  x_i^{(1)} \right) 
		= \pi(x_i).
	\end{align*}
	
	Similarly the contribution of the mechanism $f$ given a subset $A \subset V$ with \FineAny assuming the aforementioned linear mechanisms is given by:
	\begin{align*}
		C(f \mid A; V) 
		&= f^{\left(\mathbf{2}_{A \cup \{f\}}(f)\right)} \left(x_1^{\left(\mathbf{2}_{A \cup \{f\}}(x_1)\right)}, \dotsc, x_d^{\left(\mathbf{2}_{A \cup \{f\}}(x_d)\right)}\right) - f^{\left(\mathbf{2}_{A}(f)\right)} \left(x_1^{\left(\mathbf{2}_{A}(x_1)\right)}, \dotsc, x_d^{\left(\mathbf{2}_{A}(x_d)\right)}\right)\\
		&= f^{(2)} \left(x_1^{\left(\mathbf{2}_{A}(x_1)\right)}, \dotsc, x_d^{\left(\mathbf{2}_{A}(x_d)\right)}\right) - f^{(1)} \left(x_1^{\left(\mathbf{2}_{A}(x_1)\right)}, \dotsc, x_d^{\left(\mathbf{2}_{A}(x_d)\right)}\right)\\
		&= \sum_{j=1}^d \beta_j^{(2)} x_j^{(\mathbf{2}_{A}(x_j))} - \sum_{j=1}^d \beta_j^{(1)} x_j^{\left(\mathbf{2}_{A}(x_j)\right)}\\
		&= \sum_{j=1}^d x_j^{\left(\mathbf{2}_{A}(x_1)\right)} \left(\beta_j^{(2)} - \beta_j^{(1)} \right)
	\end{align*}
	Following a similar argument as before, half of the sets $A$ in the powerset $\mathcal{P}(\{f, x_1, \dots, x_d\} \setminus \{f\})$ contain $x_i$, half do not. Using the linearity of expectation property, the average contribution of $f$, over all subsets $A$, which is also its Shapley value, is given by 
	\begin{align*}
		\varphi(f) &= \sum_{j=1}^d \frac{1}{2} x_j^{(1)} \left(\beta_j^{(2)} - \beta_j^{(1)} \right) + \sum_{j=1}^d \frac{1}{2} x_j^{(2)} \left(\beta_j^{(2)} - \beta_j^{(1)} \right) = \sum_{j=1}^d \frac{\left( x_j^{(1)} + x_j^{(2)} \right)}{2} \left(\beta_j^{(2)} - \beta_j^{(1)} \right)\\
		&= \pi(f)
	\end{align*}
\end{proof}

\end{document}